\documentclass[11pt]{article}

\usepackage[margin=1in]{geometry}
\usepackage{setspace}
\usepackage{graphicx}
\usepackage{epsfig}
\usepackage{natbib}
\usepackage{bm}
\usepackage{amsmath,amssymb,amsthm}
\usepackage[ruled,linesnumbered]{algorithm2e}
\usepackage{dsfont}
\usepackage{xcolor}
\usepackage{color}
\usepackage{booktabs}
\usepackage{float}
\usepackage{microtype}
\usepackage[colorlinks,
linkcolor=blue,
anchorcolor=blue,
citecolor=blue,
]{hyperref}
\usepackage{tikz}
\usetikzlibrary{arrows.meta,positioning,shapes.geometric}

\newcommand{\ex}{{\mathsf E}}

\newcommand{\ts}{{\mathtt{T}}}
\let\mathscr\mathcal
\DeclareMathOperator*{\argmax}{arg\,max}

\theoremstyle{plain}
\newtheorem{theorem}{Theorem}
\newtheorem{proposition}{Proposition}
\newtheorem{corollary}{Corollary}

\theoremstyle{definition}
\newtheorem{assumption}{Assumption}

\newcommand{\Halmos}{\qed}
\newcommand{\ECSwitch}{\appendix}
\newcommand{\ECHead}[1]{\begin{center}#1\end{center}}
\makeatletter
\renewenvironment{proof}[1][Proof]{\par\noindent\textit{#1.}\ }{\par}
\makeatother

\allowdisplaybreaks[4]

\bibpunct[, ]{(}{)}{,}{a}{}{,}

\begin{document}

\title{Nonparametric Learning and Earning with One-Point Feedback under Nonstationarity}
\author{
Xiangyu Yang\\
School of Management, Shandong University\\
\texttt{yxy@email.sdu.edu.cn}
\and
Feng Xu\thanks{Corresponding author}, Jian-Qiang Hu\\
School of Management, Fudan University\\
\texttt{fxu25@m.fudan.edu.cn}, \texttt{hujq@fudan.edu.cn}
\and
Jiaqiao Hu\\
Department of Applied Mathematics and Statistics, Stony Brook University\\
\texttt{jiaqiao.hu.1@stonybrook.edu}
}
\date{ }
\maketitle

\begin{abstract}
	Firms increasingly rely on dynamic pricing to respond to evolving customer demand, yet in many applications they observe only the revenue generated by a single posted price in each period. At the same time, market conditions may shift gradually or abruptly due to changes in customer preferences, competition, or external shocks. These features create two intertwined challenges: learning the revenue--demand relationship from limited feedback and adapting pricing decisions to a changing environment. We study how a seller can learn and earn effectively under these constraints, without assuming a specific parametric form for demand.
	We develop a learning framework that updates prices using revenue-based gradient approximations constructed from one observation per period. To address environmental changes, we incorporate a restarting mechanism that periodically refreshes the learning process so that outdated information is discounted. When the degree of nonstationarity is unknown, we further introduce a meta-learning layer to adaptively hedge across multiple restarting schedules.
	We provide performance guarantees for our approach, showing how cumulative revenue loss relative to a fully informed benchmark depends on both the time horizon and the magnitude of market variation. Simulation experiments using synthetic and real-world data illustrate the effectiveness of the proposed procedures.
\end{abstract}


\bigskip

\section{Introduction}\label{sec:intro}

Dynamic pricing has become a widely adopted strategy in modern revenue management and online marketplaces, involving real-time adjustments of product prices in response to customer demand  \citep{DENBOER20151}. Advances in data analytics have enabled sellers to apply statistical methods to learn demand through continual price updates.
Dynamic pricing with demand learning is a canonical learning-and-earning problem, in which the decision maker faces a fundamental trade-off between exploitation, by selecting prices believed to be profitable, and exploration, by experimenting with prices to improve future decisions.
In practice, however, the market environment is often nonstationary. The revenue landscape may evolve gradually (e.g., due to shifts in willingness to pay), experience abrupt regime changes (such as promotions, competitor entry, or policy interventions), or fluctuate with calendar effects and exogenous shocks.
For instance, as reported by \emph{The Wall Street Journal} in mid-2022 \citep{WSJ2022},
major retailers such as Walmart encountered unprecedented operational challenges when a sharp rise in inflation abruptly shifted consumer spending away from discretionary goods. Such macroeconomic shocks can render historical demand models obsolete, compelling firms to swiftly implement dynamic markdown strategies while relearning consumer preferences.
Nonstationarity therefore introduces an additional challenge: an effective pricing policy must continuously adapt by selectively discounting outdated information. When combined with the exploration–exploitation trade-off, this challenge gives rise to a three-way tension among learning, revenue generation, and adaptation to evolving conditions.

The problem of learning and earning under nonstationarity has attracted considerable attention in the literature. 
A variety of modeling frameworks have been developed, including change-point detection methods, Markov process-based models, and autoregressive specifications; a comprehensive review is provided in Section~\ref{sec:literature}.
Notably, several studies have examined more general nonstationary environments that essentially subsume the aforementioned settings; cf., e.g., \cite{DENBOER2015EJOR, Keskin2017MOR, Zhu2020ICML}. 
These general models often rely on partial knowledge of the demand--price relationship. In particular, \cite{DENBOER2015EJOR} studies a demand environment with an evolving market size under the assumption of known price sensitivity,  \cite{Keskin2017MOR} consider a setting in which product demand follows a linear functional form, and \cite{Zhu2020ICML} investigate a scenario where the underlying demand process increases over time in expectation and may also experience growth in the magnitude of random fluctuations.

In this paper, we study the learning and earning problem in a general nonstationary, \emph{nonparametric} setting. Our setting requires little structural knowledge of the demand--price relationship, and the revenue can be inferred only from realized sales and therefore lacks an explicit analytical expression.
This combination of nonstationarity and model uncertainty is especially relevant in operational settings where demand can change quickly, and managers may be reluctant to commit to a potentially unreliable parametric model.
Stochastic gradient-based methods are effective approaches for revenue maximization in nonparametric dynamic pricing, in which Kiefer--Wolfowitz (KW) methods and their variants are commonly  employed to construct gradient estimates from revenue feedback \citep[see][]{Hong2020NRL, YANG2024EJOR}. However, KW methods typically require multiple revenue observations to construct a finite-difference-based gradient estimate, although the seller can post only one price in each period. Moreover, in nonstationary settings, the revenue function with respect to price may vary across periods, making a multi-point gradient approximation approach inaccurate and difficult to implement in practice.

To address the joint challenges of nonstationarity and single-point observations, we propose a hierarchical framework that integrates an appropriate ``forgetting principle'' to discard outdated information \citep{Garivier2011icalt} with a one-point (function-value) gradient estimation method. Specifically, our approach builds on a restarting mechanism and incorporates one-point feedback techniques from the online optimization literature \citep[see, e.g.,][]{la2025gradient, Lattimore_2026}. 
This design enables the algorithm to rely on only a single revenue observation per period while adapting to nonstationary environments by prioritizing more recent data.
The key idea underlying one-point feedback is the use of smoothing. Instead of optimizing the original revenue function directly, the method operates on a locally averaged surrogate obtained by perturbing the posted price and observing the resulting revenue. This smoothing step allows the construction of a nearly unbiased gradient estimator based on a single observation, which can then be incorporated into an online update rule.
Building on this idea, we first analyze a generic mirror ascent algorithm equipped with one-point feedback under mild bias--variance conditions on the gradient estimator. We then introduce a restarting mechanism that periodically resets the algorithm after a finite number of iterations, allowing it to track changes in nonstationary environments better. This mechanism is designed under the assumption that the ``variation budget'' of the underlying revenue function sequence is known. The variation budget quantifies how much the underlying revenue function can change over time, serving as a measure of the degree of nonstationarity; its formal definition is deferred to subsequent sections.
Moreover, to eliminate the need for prior knowledge of the variation budget, we introduce a meta-layer that aggregates a pool of restarting-based experts via exponential weighting \citep[see, e.g.,][]{CB2006}, thereby tracking the best restart schedule over time. Such a ``bandit-over-bandit'' strategy \citep{Cheung2022MS} couples the core restarting mirror ascent procedure with an outer expert-tracking mechanism that selects among base algorithms calibrated to different variation-budget levels. 
Collectively, these components provide a unified and modular approach to nonparametric dynamic pricing in changing environments: a base learner for local price updating, a restarting wrapper for discounting stale information, and an adaptive aggregation layer for selecting an effective forgetting pace when the environment’s variation is unknown.
 Figure~\ref{fig:framework} presents an overview of the proposed framework.

\begin{figure}[ht]
	\centering
	\resizebox{0.5\linewidth}{!}{
		\begin{tikzpicture}[
			font=\small,
			level/.style={trapezium, trapezium angle=70, draw, align=center, minimum height=1.1cm},
			level1/.style={level, minimum width=11.2cm, fill=blue!4},
			level2/.style={level, minimum width=9.4cm, fill=blue!8},
			level3/.style={level, minimum width=5.6cm, fill=blue!12},
			arr/.style={-{Latex[length=2mm]}, thick}
			]
			\node[level3] (top) {\textbf{Algorithm~\ref{alg:hedge}}\\Expert-Weighting over Restarting Online Mirror Ascent\\\emph{(meta-layer for unknown variation budget)}};
			\node[level2, below=0.55cm of top] (mid) {\textbf{Algorithm~\ref{alg:dynamic}}\\Restarting Online Mirror Ascent\\\emph{(wraps Algorithm~\ref{alg:static} via restarts)}};
			\node[level1, below=0.55cm of mid] (base) {\textbf{Algorithm~\ref{alg:static}}\\Online Mirror Ascent with One-Point Feedback\\\emph{(base learner within a batch)}};
			\draw[arr] ([xshift=7pt]top.south) -- node[midway, right]{\scriptsize meta (expert-weighting)} ([xshift=7pt]mid.north);
			\draw[arr] ([xshift=-7pt]mid.north) -- node[midway, left]{\scriptsize be invoked} ([xshift=-7pt]top.south);
			\draw[arr] ([xshift=7pt]mid.south) -- node[midway, right]{\scriptsize restarting / batching} ([xshift=7pt]base.north);
			\draw[arr] ([xshift=-7pt]base.north) -- node[midway, left]{\scriptsize be invoked} ([xshift=-7pt]mid.south);
		\end{tikzpicture}
	}
	\caption{Overview of the proposed methods: Algorithm~\ref{alg:static} (base) $\leftrightarrow$ Algorithm~\ref{alg:dynamic} (middle) $\leftrightarrow$ Algorithm~\ref{alg:hedge} (top).}
	\label{fig:framework}
\end{figure}
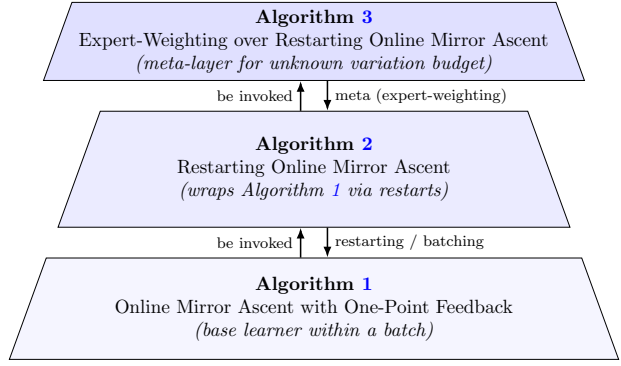

From a theoretical perspective, our framework accommodates a broad spectrum of stochastic gradient estimators with a single design point.
We rigorously derive upper bounds on the algorithm's regret (i.e., the expected loss relative to using the optimal price) that explicitly separate the contributions of the time horizon and environmental variation, and further demonstrate that similar bounds can be achieved without prior knowledge of either.
To the best of our knowledge, this is among the first works to develop a unified regret analysis for nonparametric dynamic pricing with demand learning under nonstationarity and one-point feedback.
From a numerical perspective, we conduct extensive simulation studies to demonstrate the effectiveness and superiority of the proposed algorithms under nonstationary demand scenarios. These experiments include ablation analyses, as well as evaluations on both synthetic data and the Walmart dataset.
The numerical study provides evidence along two dimensions: it isolates the contributions of restarting and meta-learning under controlled synthetic nonstationarity, and it further validates the practical relevance of the proposed framework in a data-driven pricing environment.
Our framework is well-suited to real-world dynamic pricing settings in which sellers observe only realized revenues from a single posted price and operate under evolving market conditions. 
The proposed algorithms are easy to implement, require limited structural assumptions, and accommodate various gradient estimation techniques, providing a scalable and principled approach to nonparametric learning and earning in nonstationary markets.

\textbf{Organization}:
The remainder of this paper is organized as follows. Section~\ref{sec:literature} reviews the related literature. Section~\ref{sec:problem} formulates the nonparametric dynamic pricing problem in a nonstationary environment and introduces the required concepts. Sections~\ref{sec:known} and \ref{sec:unknown} present the proposed algorithms, including the restarting mirror ascent method and the bandit-over-bandit meta-algorithm, and establish their theoretical regret guarantees. Section~\ref{sec:num} reports numerical experiments illustrating the empirical performance of the proposed methods. Section~\ref{sec:conclu} concludes with a discussion of implications and directions for future research.

\textbf{Notation}:
Throughout the paper, all vectors are column vectors unless otherwise specified.
We use $\|\cdot\|$ to denote the Euclidean norm for vectors and the spectral norm for matrices, and $\mathrm{vol}(\cdot)$ to denote the volume of a set.
The ceiling symbol $\lceil \cdot \rceil$ represents rounding up to the nearest integer.
We denote by $\sigma\{\mathscr X\}$ the smallest $\sigma$-field generated by the collection of random elements $\mathscr X$.
For two positive sequences ${v_k}$ and ${w_k}$, we write $v_k = O(w_k)$ if $\limsup_{k\to\infty} v_k/w_k < \infty$ and $v_k = o(w_k)$ if $\lim_{k\to\infty} v_k/w_k = 0$.

\section{Related Literature}\label{sec:literature}

This section provides a brief review of two streams of related literature, with an emphasis on modeling approaches to nonstationary learning and earning and bandit convex optimization techniques.

We elaborate on the streams of research mentioned in Section~\ref{sec:intro}, including change-point detection, Markov process-based models, and autoregressive specifications. 
Change-point detection provides a natural approach for identifying structural breaks in time series data. In dynamic pricing applications, this framework can be extended to detect temporal shifts in the underlying demand function; see, e.g., the single change-point setting in \cite{Besbes2011OR} and the multiple change-point formulation in \cite{Boer2020MS}.
Markov process-based models adopt a different viewpoint by positing that the underlying demand environment exhibits Markovian dynamics. For example, \cite{Aviv2005MS} assume that the sequence of demand functions evolves as a discrete-time Markov chain, where the decision maker has partial information about the prevailing state at the beginning of each period, leading to a partially observed Markov decision process formulation.
\cite{Zhang2016NRL} study a two-class revenue management admission-control problem under nonstationary demand, where customer arrivals follow nonhomogeneous Poisson processes that can be updated using available demand and exogenous information.
In a continuous-time setting, \cite{Chen2013OR} model the evolution of market size as a Gaussian process. More recently, \cite{Keskin2024OR} study dynamic pricing in a Markovian demand environment with unknown transition probabilities.
As a closely related alternative to Markovian models, autoregressive specifications capture nonstationarity by assuming that the parameters of a parametric demand function evolve according to an autoregressive process \citep{BECK2002JEDC}.
As noted earlier, all these models impose specialized structural assumptions on the demand function and can therefore be viewed as special cases of the broader framework considered in this paper, which only requires a variation budget for the demand functions.

In addition to the studies discussed above,
demand variability also arises in contextual dynamic pricing settings \citep[see, e.g.,][]{Ban2021MS, Miao2022pom, Luo2024MOR, zhao2026contextual}.
In such settings, observed contexts such as product characteristics and customer features induce heterogeneous and evolving demand environments, and the seller seeks to set personalized prices through sequential interactions with customers. 
A special class of contextual dynamic pricing captures consumers' behavioral biases through reference prices modeled as a sequence of covariates \citep{Boer2022MS}. A key characteristic is that the covariate sequence is endogenously driven by the history of posted prices.
Beyond reference-price effects, other mechanisms can also induce endogenously evolving environments. For instance, when customer patience is taken into account \citep{Zhang2022MSOM, BirgeOR2025}, forward-looking consumers may wait across multiple selling opportunities, creating intertemporal dependence between pricing decisions and demand.
Moreover, nonstationarity may arise in competitive markets with multiple sellers, where rivals' pricing adjustments influence demand dynamics; e.g., \citet{Meylahn2022MSOM} study a KW-type recursion and obtain an interesting result that the algorithm can learn to collude under self-play (ignoring the presence of competition) even though the price--demand relationship is unknown. A related learning problem in general monotone games has been investigated by \cite{Bravo2018nips, Ba2025OR}.
It can be seen that all the approaches discussed in this paragraph either explicitly incorporate covariates into the pricing policy or model environmental changes as endogenously driven by the seller’s pricing decisions; consequently, they differ fundamentally from the framework studied in this paper.
For a comprehensive overview of nonstationary learning and earning models, we refer the reader to \citet{denBoer2022book}.

From a technical standpoint, our solution approach draws on the literature on nonstationary sequential optimization and bandit convex optimization. Initiated by \cite{Zinkevich2003ICML}, adversarial online convex optimization (OCO) has become a central paradigm in the learning community and has undergone extensive development and broad generalization \citep[see][and references therein]{hazan2022introduction}.
The seminal work of \cite{Besbes2015OR} establishes a fundamental connection between adversarial OCO and nonstationary stochastic optimization, demonstrating that policies designed for the adversarial setting can be adapted to nonstationary environments through a simple restarting procedure. 
Much of the research on nonstationary stochastic optimization, including \cite{Besbes2015OR}, assumes a relatively idealized case in which certain regularity conditions on the sequence of objective functions are known \citep{Yang2016ICML, Wei2016NIPS, wei2018abruptly}.
To overcome this restriction, significant effort has been devoted to developing methods that adapt to changing environments on the fly. Techniques achieving this include the doubling trick \citep{Jadbabaie2015AISTATAS}, heuristic envelopes \citep{Besbes2019SS}, bandit corralling \citep{Luo2018COLT}, moving windows and decaying weights \citep{Keskin2017MOR},  prediction with expert advice \citep{Zhang2018ICML}, distributional change detection \citep{chen2025arxiv}, and multiscale sampling \citep{Wang2025OR}.
However, most of the existing literature focuses on settings with full-information feedback, where the objective function or its gradient is observable, or is limited to linear bandit models. The works of \cite{Zhao2021JMLR} and \cite{Wang2025OR} are most closely related to the nonparametric single-point feedback setting studied in this paper. \cite{Zhao2021JMLR} consider a noiseless objective function, and \cite{Wang2025OR} assumes strong convexity (for the minimization problem). In contrast, the present study extends the analysis to accommodate stochastic noise and requires only the concavity of the revenue function.

\section{Problem Formulation}\label{sec:problem}

We consider a monopolist firm (seller) that sells $d$ distinct products over $T$ discrete time periods. At the beginning of each time period $t\in \{1,\ldots,T\}$, the seller first offers a price vector $\tilde x_t\in \mathcal K \subset \mathbb R^d_{+}$, where $\mathcal K$ is a convex body (i.e., a compact, convex set with non-empty interior) contained in the $d$-dimensional Euclidean space with nonnegative coordinates. A demand realization in response to the posted price is then observed by the seller. In the absence of inventory constraints (so that all demand can be satisfied) and with zero marginal costs, this setting is equivalent to the seller receiving a revenue feedback $\phi_t(\tilde x_t):=r_t(\tilde x_t)+\xi_t$, where $r_t(\cdot)$ denotes the expected revenue function in period $t$, and $\xi_t$ is a  zero-mean stochastic disturbance that may depend on past prices. Let $x_t^*$ be a maximizer of $r_t(\cdot)$ over $\mathcal K$. Throughout this paper, we impose the following regularity conditions on the revenue functions:

\begin{assumption}[Revenue Function]
	\label{asp:rf}
For each period $t$, the revenue function $r_t(\cdot)$ is concave and continuously differentiable on $\mathcal K$, with $L$-Lipschitz continuous gradients; that is, there exists a constant $L>0$ such that, for all $x,y\in\mathcal K$,
\[
\|\nabla r_t(x)-\nabla r_t(y)\|\le L\|x-y\|.
\]
Moreover,  $x_t^*$ belongs to a convex body $\Theta$ satisfying $\Theta\subset\mathcal K^\circ$, where $\mathcal K^\circ$ denotes the interior of $\mathcal K$.

\end{assumption}
Assumptions on the concavity/convexity and smoothness of the objective function are common in the bandit convex optimization literature \citep[see, e.g.,][]{Besbes2015OR, Wang2025OR}. 
The requirement that the maximizer be an interior point is also standard in the pricing literature and is typically innocuous in practical applications. 
It is worth noting that we introduce the set $\Theta$, which lies strictly away from the boundary of $\mathcal K$, to ensure that the revenue function is evaluated only at valid inputs in each period. This argument will be formalized in the sequel; see Assumption~\ref{asp:ut}.
Indeed, both the subsequent algorithm design and theoretical analysis are primarily carried out over $\Theta$.

The widely used linear demand model \citep[cf., e.g.,][]{Besbes2015MS} satisfies the above assumption. Nevertheless, in the nonparametric setting, the seller does not have access to the explicit forms of the expected revenue functions $\{r_t(\cdot)\}_{t=1}^T$.
Operationally, we treat the revenue environment as exogenously evolving over time and not controlled by the seller's current pricing decision. 
From a game-theoretic perspective, this uncertainty can be interpreted as an oblivious adversary who secretly selects the entire revenue sequence at the start of the game. Consequently, to earn higher cumulative revenue, the seller must learn the underlying revenue functions over time. 
Formally, a learning policy is a sequence of mappings that generates the price vector  at each period.
The learning policy is required to be nonanticipating, that is, each pricing decision (except possibly the initial price) may depend only on historical prices and observed revenues, and may incorporate randomized strategies through dependence on exogenous randomness.
To gauge the performance of the learning policy, we adopt the standard notion of \emph{dynamic regret}. This benchmark captures the revenue loss from not pricing as well as a fully informed seller who knows the period-specific revenue landscape in advance. It is defined as
\begin{equation}\label{eq:d-reg}
	\mathrm{D\text{-}Reg}(T)=
	\sum_{t=1 }^{T}r_t(x^*_t) - 
	\mathsf E\left[ \sum_{t=1 }^{T}r_t(\tilde x_t)  \right].
\end{equation} 
Note that under oblivious nonstationarity, the choice of $x_t^*$ (or any price in period $t$) does not influence future revenue functions $\{r_{t+1}(\cdot), \ldots, r_T(\cdot)\}$, which are considered to be exogenously determined by an adversary at the beginning of the horizon.
The first term on the right-hand side of (\ref{eq:d-reg}) represents the revenue that would be collected if the seller knew the entire sequence of revenue functions in advance (i.e., under clairvoyance). The expectation in the second term is taken over both the stochasticity in revenue outcomes and any randomization in the pricing policy. For brevity, we omit explicit specification of the underlying probability measure with respect to which expectations are taken in what follows.
Our objective is to design an effective learning policy that achieves a slow growth rate of dynamic regret over time.

 However, it is well known that sublinear dynamic regret, i.e., $\mathrm{D\text{-}Reg}(T)=o(T)$, cannot be attained for an arbitrary sequence of revenue functions \citep{Besbes2015OR}.  Hence, it is often necessary to impose certain regularity conditions on the sequence of revenue functions to obtain meaningful dynamic regret bounds. In this work, we leverage a commonly used regularity measure called \emph{path variation}, which constrains the variation of the maximizers of the function sequence \citep[see, e.g.,][]{Zinkevich2003ICML, Yang2016ICML}. Formally, there exists a variation budget $V_T$ such that
\[
\max_{ \{x_t^*\in \Theta^*_t\}_{t=1}^{T} }\sum_{t=2}^{T}\|x^*_{t}-x^*_{t-1}\| \leq V_T,
\] 
where $\Theta^*_t :=\argmax_{x\in \mathcal K}r_t(x)$.
Notice that, under the condition $x^*_t \in \Theta$ for all $t$ (Assumption~\ref{asp:rf}), the variation metric admits a natural upper bound of $T\mathrm{diam}(\Theta)$, where $\mathrm{diam}(\Theta):= \sup_{x,x' \in \Theta} \|x - x'\| $ denotes the diameter of $\Theta$. Thus, without loss of generality, we normalize $\Theta$ such that $\mathrm{diam}(\Theta) \equiv \sqrt{d}$ and assume that $\sqrt{d} \leq V_T \leq \sqrt{d}T.$
In a nonstationary environment, $V_T$ constitutes a pivotal  measure of the extent of environmental variation and plays a fundamental role in pricing policy design; yet, its value may or may not be known a priori.
Accordingly, the following sections, Sections~\ref{sec:known} and \ref{sec:unknown}, present principled approaches for the cases in which $V_T$ is known and unknown, respectively.

\section{Known $V_T$: Restarting Mirror Ascent Bandit} \label{sec:known}

As discussed in Section~\ref{sec:intro}, the seller faces a trilemma involving exploration, exploitation, and adaptation to changes, where the need for adaptation distinguishes nonstationary settings from stationary ones. Building on the seminal work of \cite{Besbes2015OR} on handling nonstationarity, we develop a restarting framework based on mirror ascent to prevent reliance on potentially obsolete observations.
In particular, a key insight gleaned from \cite{Besbes2015OR} is that a policy with strong  performance against the single best action in hindsight, i.e., the optimal static decision chosen after observing the entire sequence of the adversary’s functions, can be adapted via restarting to achieve favorable dynamic regret guarantees.
In Section~\ref{sec:oco}, we formally introduce our algorithm within the adversarial OCO framework to enable analysis relative to the single best price in hindsight. Section~\ref{sec:ds} then extends these results to the nonstationary stochastic setting.

\subsection{Adversarial OCO Setting}\label{sec:oco}

Let us define the \emph{static regret}, which quantifies the difference between the total revenue obtained by the seller and that of the best fixed price in hindsight, as follows:
\begin{align}\label{eq:s-reg}
	\mathrm{S\text{-}Reg}(T)=
	\max_{x\in \Theta}\sum_{t=1 }^{T}r_t(x) -
	\mathsf E\left[ \sum_{t=1 }^{T}r_t(\tilde x_t)  \right].
\end{align}
Compared with the dynamic regret defined in \eqref{eq:d-reg}, the main distinction lies in the order of the summation and maximization operators in the clairvoyant benchmark. Hence, static regret is a weaker performance criterion, since the benchmark associated with dynamic regret is always at least as strong as its static-regret counterpart.
Another minor difference lies in the domain over which the static-regret benchmark is defined: we use $\Theta$ rather than $\mathcal K$. This choice facilitates the subsequent technical analysis and does not affect our main conclusion, which concerns dynamic regret, since Assumption~\ref{asp:rf} ensures that the period-wise optimal solution lies in $\Theta$.

Motivated by the bandit convex optimization framework of \cite{Hu2016AISTATS}, we develop our algorithm by integrating online mirror ascent with a general gradient estimation scheme. The adoption of mirror ascent is driven by the need to balance local improvement with stability under noisy one-point feedback. In particular, the mirror ascent step solves a proximal optimization problem that balances ascent along the estimated gradient and proximity to the current iterate.
The update direction is determined by the estimated gradient, encouraging movement toward regions of higher revenue, while the regularizer term penalizes large deviations from the current iterate.  This is desirable in our setting, where gradient estimates are constructed from a single noisy observation and overly aggressive updates can degrade performance.
Throughout the paper, we denote by $\psi(\cdot)$ the regularization function and assume that it is $\alpha$-strongly convex. Under the standard choice of the Euclidean regularizer $\psi(x)=\|x\|^2/2$ as the mirror map, mirror ascent reduces to the conventional (projected) gradient ascent method.

Let $\tau$ denote the time horizon in the adversarial OCO setting, where $1 \leq \tau \leq T$, and
let $\delta$ and $\eta$ be two positive algorithmic parameters. The complete procedure is described in Algorithm~\ref{alg:static} below.
\begin{algorithm}[htbp]  
	\caption{Online Mirror Ascent with One-Point Feedback}
	\label{alg:static}
	\KwIn{time horizon $\tau$; perturbation size $\delta$; step size $\eta$;	
	}
	\BlankLine
	Initialize the time index by setting $t = 1$ and select an initial point $x_1 \in \Theta$;
	
	\While{$t \leq \tau$}
	{
		Generate a set of jointly distributed random variables $(\tilde U_t, \bar U_t)\in \mathbb R^d \times \mathbb R^d$;
		
		\uIf{$t=\tau$}
		{
			Set the price $\tilde x_t=x_t+\delta \tilde U_t$, assign $G_t = 0$, and terminate;     
			\tcp{$G_{\tau}$ is set to zero for technical reasons and is not required in practice.}
		}
		\Else{
			Set the price $\tilde x_t=x_t+\delta \tilde U_t$ and 
			construct a stochastic gradient estimator $G_t$ using one-point feedback of the form
			$$
			G_t=\frac{\phi_t(\tilde x_t)}{\delta }\bar U_t.
			$$
			
			Update $x_t$ for the next period via the mirror ascent step
			$$
			x_{t+1}=\argmax_{x\in \Theta} \left\{\eta  
			G_t^{\mathtt T}x - B_{\psi }(x,x_t)
			\right\},
			$$
			where $B_{\psi}(x,y):=\psi(x)-\psi(y)-\nabla \psi (y)^{\mathtt T}(x-y)$ for all $x,y \in \Theta $, and $\psi(\cdot)$ is an $\alpha$-strongly convex function on $\Theta$;
			
			Increment the time index by setting $t = t+1$;
		}
	}
\end{algorithm}

In a nonstationary environment, the seller can submit only a single price query in each period to obtain an appropriate gradient estimate, since the underlying functions may change over time. Therefore, we rely on a one-point feedback approach to construct the gradient estimator $G_t$ in Step~7 of Algorithm~\ref{alg:static}. Despite this restriction, $G_t$ admits a broad and flexible formulation that encompasses many widely used estimation schemes as special cases, including estimators based on spherical smoothing \citep{Flaxman2005SODA}, Gaussian smoothing \citep{Gao2022ICML},  coordinate-wise perturbations \citep{Besbes2015OR}, and simultaneous perturbations \citep{Spall1997Automatica}.
In this study, we do not confine our analysis to a specific form of $G_t$. Instead, we require it to satisfy the conditions stated in Assumptions~\ref{asp:ut} and~\ref{asp:oracle}.

\begin{assumption}[Perturbation]
	\label{asp:ut}
	For every time period $t$, it holds that $x_t+\delta \tilde U_t\in \mathcal K$. 
	Furthermore, the random vector $\tilde U_t$ has zero conditional mean and bounded second moment; that is, there exists a constant $C_u>0$ such that, with probability one,
	\[
	\mathsf E\left[ \tilde U_t | \mathcal F_t   \right]=0 \text{ and }
	\mathsf E\left[   \|\tilde U_t\|^2  | \mathcal F_t \right]\leq C_u,
	\]
	where 
	$
	\mathcal F_t := \sigma\{x_1, \tilde U_1,\bar U_1, \xi_1, \ldots,  \tilde U_{t-1}, \bar U_{t-1}, \xi_{t-1}\}
	$ denotes the filtration capturing all information available up to $t$.
\end{assumption}

\begin{assumption}[Bias--Variance]
	\label{asp:oracle}
	There exist constants $C_1,C_2,p,q>0$ such that, with probability one,
	\begin{equation*}
		\| \nabla r_t(x_t) - \mathsf E\left[ G_t | \mathcal F_t \right]  \| 
		\leq C_1 \delta^p  \text{ and }
		\mathsf E\left[
		\| G_t-\mathsf E[G_t|\mathcal F_t]  \|^2 | \mathcal F_t\right]
		\leq C_2 \delta^{-q}.
	\end{equation*}
	
\end{assumption}
 Assumption~\ref{asp:ut} is not restrictive and guarantees the validity of the stochastic gradient construction.
The constants $C_1,C_2,p,q$ in Assumption~\ref{asp:oracle} are properties of the gradient estimation scheme and the underlying problem (e.g., smoothness of $r_t(\cdot)$), and do not depend on the algorithmic parameters $\delta$ and $\eta$.
Concrete examples of gradient estimators satisfying Assumption~\ref{asp:oracle} can be derived under additional regularity conditions.
For instance, following an analysis similar to that of the stochastic gradient estimator proposed by \cite{Spall1997Automatica}, if   $\tilde U_t$ is symmetrically distributed with respect to the origin and $\bar U_t = h(\tilde U_t)$, where $h:\,\mathbb R^d \to \mathbb R^d$ is an odd function satisfying $\ex[\bar U_t]=0$ and $\ex[\bar U_t \tilde U_t^{\ts}] = \mathds{1}$ with $\mathds 1$ being the identity matrix, then the resulting gradient estimator satisfies Assumption \ref{asp:oracle} with $p=1,q=2$. Alternatively, building on the seminal work on spherical smoothing \citep{Flaxman2005SODA}, if   $\bar U_t = \mathbb N_W(\tilde U_t)\mathrm{vol}(\mathbb S_W)/\mathrm{vol}(W)$, where $W$ is a centrally symmetric convex body with boundary $\mathbb S_W$, $\tilde U_t$ is uniformly distributed on $\mathbb S_W$, and $\mathbb N_W(\tilde U_t)$ denotes the outward normal vector of $\mathbb S_W$ at $\tilde U_t$, then the resulting gradient estimator satisfies Assumption \ref{asp:oracle} with $p=q=2$.

Under the above assumptions, we establish the following proposition, which provides a general upper bound on the static regret of Algorithm~\ref{alg:static}. The analysis utilizes techniques developed for deterministic optimization via the mirror descent algorithm \citep{Beck2003ORL}, with appropriate modifications tailored to our online stochastic setting; see Section~\ref{ec:sec3} of the E-Companion for a detailed proof.
\begin{proposition}\label{pps:static}
	Suppose Assumptions~\ref{asp:rf}--\ref{asp:oracle} hold. 
	Define  $B=\sup_{x,y\in \Theta}B_{\psi}(x,y)$ and $C_r=\sup_{x\in \Theta}\| \nabla r_t(x) \|^2 $ for all $t$. 
	Then the static regret of Algorithm~\ref{alg:static} satisfies
	\begin{equation}\label{eq:oco_reg}
		\max_{x\in \Theta}\sum_{t=1}^{\tau}r_t(x)- 
		\mathsf E\left[\sum_{t=1}^{\tau}
		r_t(\tilde x_t)
		\right] 
		\leq
		\frac{B}{\eta}+\frac{\tau\eta}{\alpha}
		\left(C_1^2\delta^{2p}+\frac{C_2}{2}\delta^{-q} +C_r \right)
		+\tau  C_1\delta^p\sqrt{d}
		+\frac{\tau LC_u}{2}\delta^2.
	\end{equation}
\end{proposition}

Proposition~\ref{pps:static} clarifies the basic trade-off created by one-point feedback learning. The regret bound contains terms associated with optimization stability, smoothing bias, and estimation variance. Intuitively, more aggressive perturbation improves the informativeness of the revenue signal but also introduces bias, whereas smaller perturbation reduces bias but makes learning noisier. In operational terms, the result formalizes a familiar experimentation trade-off: insufficient experimentation slows learning, while excessive experimentation sacrifices current revenue.

Building on Proposition~\ref{pps:static}, we further tune the parameters $\eta$ and $\delta$ to optimize the upper bound on the right-hand side of \eqref{eq:oco_reg}, with all other problem- and estimator-dependent constants held fixed.
For simplicity, we assume $\delta \in (0,1)$ and retain only the dominant terms, approximating the bound by
\[
\frac{B}{\eta} + \frac{\tau \eta C_2}{2\alpha} \delta^{-q} + \tau \hat{C} \delta^{\hat{p}},
\]
where $\hat C:=\max\{ C_1\sqrt{d}, LC_u/2 \}$ and $\hat p:=\min\{ p,2 \}$. Minimizing this approximate bound yields the result in Corollary~\ref{crl:static}. The use of the $O(\cdot)$ notation in the corollary is intended to avoid an overly lengthy statement. 

The proof of the corollary follows from routine calculations and is therefore omitted.

\begin{corollary}\label{crl:static}
	Suppose all conditions in Proposition~\ref{pps:static} hold. 
	Define $\hat C=\max\{ C_1\sqrt{d}, LC_u/2 \}$ and $\hat p=\min\{ p,2 \}$.
	If the algorithmic parameters are chosen to satisfy
	$$
	\eta = \sqrt{2}
	\left(\frac{\alpha}{C_2}\right)^{\frac{\hat p}{2\hat p+q} }
	\left( \frac{q}{\hat p\hat C} \right)^{ \frac{q}{2\hat p+q} }
	B^{\frac{\hat p +q}{2\hat p +q} }
	\tau^{-\frac{\hat p+q}{2\hat p+q} }  \text{ and }
	\delta=\left( \frac{C_2B}{2\alpha} \right)^{\frac{1}{2\hat p+q} }
	\left( \frac{q}{\hat C \hat p} \right)^{\frac{2}{2\hat p+q} }
	\tau^{-\frac{1}{2\hat p+q} }< 1,
	$$
	then the static regret of Algorithm~\ref{alg:static} is of order
	$
	O(\tau^{(\hat{p}+q) / (2\hat{p}+q)}).
	$
	Moreover, if all constants $L,C_u,C_1,C_2,C_r,B$ scale at most polynomially with the dimension $d$, denoted as $O(\mathrm{poly}(d))$, then the hidden constant in  $O(\tau^{(\hat{p}+q) / (2\hat{p}+q)})$ is also of order $O(\mathrm{poly}(d))$. Consequently, the static regret of Algorithm~\ref{alg:static} satisfies
	$
	O(\mathrm{poly}(d)\tau^{(\hat{p}+q) / (2\hat{p}+q)}).
	$
\end{corollary}

We remark that, in the special case where Assumption~\ref{asp:oracle} holds with $p=q=2$ (e.g., using spherical smoothing), the static regret is of order $O(\tau^{2/3})$. This recovers the result of \cite{Saha2011AISTATS} and matches the lower bound established in the literature \citep[see][Theorem~4]{Hu2016AISTATS}.
On the other hand, the $O(\mathrm{poly}(d))$ scaling assumption on all constants is typically satisfied for a broad class of practical models. In particular, under the standard choice of the mirror map given by the Euclidean regularizer $\psi(x)=\|x\|^2/2$, the Bregman diameter $B=\sup_{x,y\in\Theta} B_\psi(x,y)$ scales as $O(d)$ whenever the set $\Theta$ has diameter $\sqrt{d}$.
In addition, for one-point gradient estimators based on spherical or Gaussian smoothing, as well as simultaneous perturbation schemes, the corresponding bias and variance constants $C_1$ and $C_2$, together with the perturbation second-moment bound $C_u$, scale at most polynomially with the dimension $d$. Finally, in typical pricing optimization models with smooth concave revenue functions defined on bounded domains, both the smoothness constant $L$ and the gradient bound $C_r=\sup_{x\in\Theta}\|\nabla r_t(x)\|^2$ are also polynomial in $d$.\footnote{Although all constants scale polynomially with the dimension, we acknowledge that the polynomial degree may be relatively high; such dependence is often unavoidable when using gradient estimation with one-point feedback. We do not pursue this issue further and refer interested readers to the comprehensive book by \cite{Lattimore_2026}.}

\subsection{Nonstationary Stochastic Setting}\label{sec:ds}

To analyze the dynamic regret, we ``port over'' results obtained in the adversarial OCO setting to the nonstationary stochastic setting through a simple restarting procedure \citep{Besbes2015OR}.
Specifically, we partition the time horizon $T$ into $\lceil T/\tau \rceil$ batches of equal length $\tau$ (except possibly the final batch, which can be less than $\tau$ periods). Within each batch, we run Algorithm~\ref{alg:static} developed in the previous subsection. Thus, the resulting algorithm can be viewed as restarting Algorithm~\ref{alg:static} at the beginning of each batch.
The algorithm description is provided in Algorithm~\ref{alg:dynamic}.
\begin{algorithm}[htbp]  
\caption{Restarting Online Mirror Ascent with One-Point Feedback}
\label{alg:dynamic}
\KwIn{time horizon $T$; batch size $\tau$;
}
\BlankLine

\For{$b=1,\ldots, \lceil T/\tau\rceil  $}{
	
	Choose the perturbation size $\delta_b$  and step size $\eta_b$ for the $b$-th batch, set the time index $t = (b-1)\tau + 1$, and select an arbitrary initial point $x_{(b-1)\tau+1} \in \Theta$ for the current batch;
	
	\While{$t \leq \min\{b\tau, T \}$}
	{
		Generate a set of jointly distributed random variables $(\tilde U_t, \bar U_t)\in \mathbb R^d \times \mathbb R^d$;
		
		\uIf{$t=T$}
		{
			Set the price $\tilde x_t = x_t + \delta_b \tilde U_t$, assign $G_t = 0$, and terminate; 
		}
		\uElseIf{$t = b\tau$}
		{
			Set the price $\tilde x_t = x_t + \delta_b \tilde U_t$, assign $G_t = 0$, and break the while loop;
		}
		\Else{
			Set the price $\tilde x_t=x_t+\delta_b \tilde U_t$ and 
			construct a stochastic gradient estimator $G_t$ using one-point feedback of the form
			$$
			G_t=\frac{\phi_t(\tilde x_t)}{\delta_b }\bar U_t.
			$$
			
			Update $x_t$ for the next period via the mirror ascent step
			$$
			x_{t+1}=\argmax_{x\in \Theta} \left\{\eta_b  
			G_t^{\mathtt T}x - B_{\psi }(x,x_t)
			\right\},
			$$
			where $B_{\psi}(x,y):=\psi(x)-\psi(y)-\nabla \psi (y)^{\mathtt T}(x-y)$ for all $x,y \in \Theta $, and $\psi(\cdot)$ is an $\alpha$-strongly convex function on $\Theta$;
			
			Increment the time index by setting $t = t+1$;
		}
	}  
}

\end{algorithm}

Recall that $r_t(\cdot)$ is continuously differentiable on a compact set for all $t$ (Assumption~\ref{asp:rf}), so it is Lipschitz continuous. 
Now, we  formally connect the performance of Algorithm~\ref{alg:dynamic} relative to the dynamic benchmark in the nonstationary stochastic setting with that of Algorithm~\ref{alg:static} relative to the single best action in the adversarial OCO setting, as demonstrated in Theorem~\ref{the:dyn_reg}. The proof follows a decomposition similar to that in \cite[][Proposition~2]{Besbes2015OR}, in which the dynamic regret is divided into batches. Within each batch, the regret consists of two components: a static regret term and a term capturing the performance gap between the single best price benchmark and the dynamic benchmark. The complete proof is deferred to Section~\ref{ec:sec3} of the E-Companion.

\begin{theorem}\label{the:dyn_reg}
Suppose all conditions in Proposition~\ref{pps:static} hold.
Let $L_r$ be the uniform Lipschitz constant of the sequence of functions $\{r_t(\cdot)\}_{t=1}^T$. 
Then the dynamic regret of  Algorithm~\ref{alg:dynamic}  satisfies 
\begin{equation*}
	\sum_{t=1}^{T}r_t(x_t^*)- \mathsf E\left[\sum_{t=1}^{T}r_t(\tilde x_t)		\right] 
	\leq \sum_{b=1}^{\lceil T/\tau \rceil}U_{b}
	+L_r\tau V_T,
\end{equation*}
where $U_{b}$ denotes the upper bound appearing on the right-hand side of \eqref{eq:oco_reg}, with $\delta$ and $\eta$ replaced by $\delta_b$ and $\eta_b$, respectively.
\end{theorem}
Theorem~\ref{the:dyn_reg} shows how restarting transforms a within-batch learning guarantee into a dynamic regret bound in a nonstationary environment. The resulting decomposition highlights two key components: the learning cost incurred within each batch and the adaptation cost induced by environmental drift across periods. This characterization underscores a central design challenge in nonstationary pricing: how to balance within-batch learning with cross-batch adaptation. If the seller relies too heavily on older observations, pricing decisions may react too slowly to current market conditions. If the seller restarts too frequently, however, too little information is accumulated to learn effectively from noisy one-point revenue observations.

Similar to the analysis in Corollary~\ref{crl:static}, the upper bound established in Theorem~\ref{the:dyn_reg} can be further optimized by appropriately tuning the parameters. In particular, we set $\delta_b$ and $\eta_b$ according to the specifications of  $\delta$ and $\eta$ in Corollary~\ref{crl:static}, and then minimize with respect to $\tau$, leading to the following result (Corollary~\ref{crl:dynamic}); the proof is provided in Section~\ref{ec:sec3} of the E-Companion.

\begin{corollary}\label{crl:dynamic}
Suppose all conditions in Theorem~\ref{the:dyn_reg} hold. Define 
$\hat C=\max\{ C_1\sqrt{d}, LC_u/2 \}$ and $\hat p=\min\{ p,2 \}$.
Let $$\tau =\left\lceil \left(\frac{\sqrt{d}T}{V_T}\right)
^{\frac{2\hat p+q}{3\hat p +q} } \right\rceil.$$
All constants $L, L_r, C_u,C_1,C_2,C_r,B$ are assumed to scale at most polynomially with the dimension $d$.
If  the algorithmic parameters are chosen to satisfy
$$
\eta_b = \sqrt{2}
\left(\frac{\alpha}{C_2}\right)^{\frac{\hat p}{2\hat p+q} }
\left( \frac{q}{\hat p\hat C} \right)^{ \frac{q}{2\hat p+q} }
B^{\frac{\hat p +q}{2\hat p +q} }
\tau_b^{-\frac{\hat p+q}{2\hat p+q} }  
\text{ and }
\delta_b=\left( \frac{C_2B}{2\alpha} \right)^{\frac{1}{2\hat p+q} }
\left( \frac{q}{\hat C \hat p} \right)^{\frac{2}{2\hat p+q} }
\tau_b^{-\frac{1}{2\hat p+q} }< 1,
$$
where $\tau_b$ denotes the size of the $b$-th batch, which may differ from $\tau$ only for the final batch,
then the dynamic regret of Algorithm~\ref{alg:dynamic} is of order

$
O( 
\mathrm{poly}(d)
T^{(2\hat p+q)/(3\hat p+q)  }
V_T^{\hat p/(3\hat p+q)  }
).
$
\end{corollary}
The specification of $\tau$ in Corollary~\ref{crl:dynamic} identifies the restart frequency when the variation budget is known. In particular, more volatile environments call for more frequent forgetting of past observations, while more stable environments justify longer learning windows. This insight is  relevant for operational settings in which managers must decide how much historical data should influence current pricing decisions.
Moreover, the result of Corollary~\ref{crl:dynamic} shows that the dynamic regret depends on a suitable measure of environmental variation rather than solely on the time horizon, thereby clearly separating the contributions of $T$ and the variation budget $ V_T$. In the special case $p = q = 2$, our algorithm achieves the same $T^{3/4}$ dependence as established in the literature on nonstationary online learning with one-point feedback; see, e.g., \cite{Chen2019ieee, Zhao2021JMLR}. (The polynomial dependence on $V_T$ differs due to the use of distinct variation measures.)

\section{Unknown $V_T$: Bandit over Bandit} \label{sec:unknown}

From Theorem~\ref{the:dyn_reg} and Corollary~\ref{crl:dynamic}, we observe that the ``optimal'' parameter configurations of the proposed algorithm require prior knowledge of the path variation budget $V_T$, which may be unavailable for real implementations. In this section, we overcome this limitation by adopting the bandit-over-bandit strategy \citep[see][and the references therein]{Cheung2022MS}, that is, by constructing a meta-layer atop Algorithm~\ref{alg:dynamic} to ``hedge'' the unknown quantities.
This approach is intuitively related to the idea of \emph{online ensembles}, which perform a grid search over candidate parameters by maintaining them in parallel and employing expert-weighting algorithms to combine predictions and track the best-performing parameter \citep{Zhao2021JMLR}. Motivated by this insight and building on the construction of \cite{Zhao2021JMLR}, we leverage the adaptive online learning algorithm of \cite{Erven2021JMLR} as a meta-algorithm and combine it with the restarting online mirror ascent procedure described above.

To elaborate, we maintain a pool of $N$ ``experts,'' where expert $i\in \{1,\ldots,N\}$ runs Algorithm~\ref{alg:dynamic} with a candidate restarting batch size $\tau^{(i)}$. At each period $t$, the meta-algorithm aggregates all experts' recommendations $\{x_t^{(i)}\}_{i=1}^N$ into a single decision $x_t$ according to the current expert weights $\{w_t^{(i)}\}_{i=1}^N$. 
Then, as described in the previous section, the seller posts a price $\tilde x_t = x_t + \delta \tilde U_t$, observes the resulting revenue, and constructs a one-point feedback stochastic gradient estimator $G_t$. 
Importantly, the experts do not interact with the market separately: in each period, the market generates only one feedback observation, corresponding to the single posted price $\tilde{x}_t$, and the algorithm accordingly constructs only one gradient estimate $G_t$.
This shared gradient signal is broadcast to all experts to update their internal states, while the meta-layer updates the expert weights via exponential reweighting based on the experts' estimated performance. The complete procedure is summarized in Algorithm~\ref{alg:hedge}.

\begin{algorithm}[htbp]
\caption{Expert-Weighting over Restarting Online Mirror Ascent with One-Point Feedback}
\label{alg:hedge}
\KwIn{time horizon $T$, number of experts $N$, candidate restarting batch sizes $\{\tau^{(1)},\ldots,\tau^{(N)}\}$, and meta-algorithm learning rate $\varepsilon$.}
\BlankLine

Set $t=1$, and for each $i \in \{1,\ldots,N\}$, initialize the $i$-th expert algorithm with the starting point $x_t^{(i)}$ and the input restarting batch size $\tau^{(i)}$, and assign the initial weight
\[
w_1^{(i)} = \frac{N+1}{N}\frac{1}{i(i+1)};
\]

\While{$t \leq T$}{
	Receive the decisions $x_t^{(i)}$ from all experts $i \in \{1,\ldots,N\}$ and aggregate them as
	\[
	x_t = \sum_{i=1}^N w_t^{(i)} x_t^{(i)};
	\]
	
	Generate jointly distributed random variables $(\tilde U_t, \bar U_t) \in \mathbb{R}^d \times \mathbb{R}^d$;
	
	\uIf{$t = T$}{
		Set the price $\tilde x_t = x_t + \delta \tilde U_t$, set $G_t = 0$, and terminate;
	}
	\Else{
		Set the price $\tilde x_t = x_t + \delta \tilde U_t$, and construct a stochastic gradient estimator $G_t$ using one-point feedback of the form
		\[
		G_t = \frac{\phi_t(\tilde x_t)}{\delta} \bar U_t;
		\]
		
		Broadcast $G_t$ to all experts as a shared gradient signal;
		
		For each expert $i \in \{1,\ldots,N\}$, obtain $x_{t+1}^{(i)}$ via mirror ascent with step size $\eta^{(i)}$ and gradient $G_t$, applied within each batch of size $\tau^{(i)}$;
		\tcp{This is regarded as each expert executing one step of its own instance of Algorithm~\ref{alg:dynamic}.
		}

		Update the expert weights according to
		\[
		w_{t+1}^{(i)} =
		\frac{w_t^{(i)} \exp\left(\varepsilon G_t^{\mathtt T} x_t^{(i)}\right)}
		{\sum_{j=1}^N w_t^{(j)} \exp\left(\varepsilon G_t^{\mathtt T} x_t^{(j)}\right)},
		\qquad \forall i \in \{1,\ldots,N\};
		\]
		
		Increment the time index by setting $t = t+1$;
	}
}

\end{algorithm}

To facilitate the analysis of the dynamic regret of the proposed algorithm under an unknown variation budget, we introduce an additional requirement stated in Assumption~\ref{asp:bound}. The boundedness assumption is mild and holds for widely used gradient approximation schemes such as spherical smoothing and simultaneous perturbation methods.
Also, the sub-Gaussian noise condition is standard in the stochastic optimization literature and accommodates noise distributions frequently encountered in practice, including Gaussian and bounded random variables.

\begin{assumption}\label{asp:bound}
For all $t$, $\bar U_t$ is bounded with $B_v:=\max_{t}\|\bar U_t\|$, and 
there exists a positive constant $\sigma_{\xi}$ such that with probability one, 
\[
\mathsf E\left[ \exp\left( \lambda \xi_t \right) \mid \mathscr{F}_t \right]
\leq
\exp\left( \lambda^2\sigma_{\xi}^2/2 \right), \quad \text{for all } \lambda \in \mathbb{R},
\]
where $\mathscr F_t:=\sigma\{\{x_t^{(i)}, \, i=1,\ldots,N\}, \tilde U_1,\bar U_1, \xi_1, \ldots, \tilde U_{t-1}, \bar U_{t-1}, \xi_{t-1}, \tilde U_{t},\bar U_{t}\}$ denotes the information available immediately prior to observing the revenue in period $t$.

\end{assumption}

Through appropriately choosing the input parameters, we obtain Theorem~\ref{the:uno_dynamic}, which establishes an upper bound on the dynamic regret of the proposed algorithm.
The reasoning behind the specific parameter selections is detailed in the proof, presented in Section~\ref{ec:sec4} of the E-Companion.
This proof builds on the analysis of Theorem~3 in \cite{Zhao2021JMLR}, while introducing new techniques that extend their results; their analysis considers a deterministic setting, whereas ours addresses a stochastic environment.

\begin{theorem}\label{the:uno_dynamic}
Suppose Assumptions~\ref{asp:rf}--\ref{asp:bound} hold.  
Define  $B=\sup_{x,y\in \Theta}B_{\psi}(x,y)$, $C_r=\sup_{x\in \Theta,\,1\leq t\leq T}\| \nabla r_t(x) \|^2$ and $B_r=\sup_{x \in \Theta,\,1\leq t\leq T}|  r_t(x) |$. 
Assume $T\geq 14$, $N=\lceil\log_2  \lceil T^{2/3} \rceil \rceil+ 1$, and $2B_r\sqrt{T}\leq \sigma_{\xi}$. 	
For each expert $i \in \{1,\ldots, N\}$, let
$ \tau^{(i)} = 2^{i-1}$, $\eta^{(i)}=\sqrt{2\alpha B/((2C_1^2\delta^{2p} +C_2\delta^{-q}+2C_r )\tau^{(i)})}$, and $\varepsilon = \delta/( B_v\sigma_{\xi}\sqrt{d T} )	$. Then the dynamic regret of Algorithm~\ref{alg:hedge} satisfies
\begin{align*}
	\mathsf E\left[\sum_{t=1}^T
	\left( r_t(x^*_t)-r_t(\tilde x_t) \right)
	\right]&\leq
	\left(4d^{-1/3}\sqrt{\frac{B(2C_1^2\delta^{2p} +C_2\delta^{-q}+2C_r )}{\alpha} }
	+2d^{1/3}\left( C_1\delta^p+\sqrt{C_r}\right)\right)
	T^{2/3}V_T^{1/3}\\
	&\quad
	+\frac{ \sigma_{\xi}B_v\sqrt{d}}{\delta}\sqrt{T}
	\left(
	1+2\ln\left( \log_2\lceil(\sqrt{d}T/V_T)^{2/3}\rceil+2 \right)
	\right) +\left(2 C_1 \delta^p \sqrt{d}+\frac{LC_u}{2}\delta^2\right)T.
\end{align*}
\end{theorem}
Theorem~\ref{the:uno_dynamic} shows that the algorithm can adapt to unknown environmental variation without requiring the seller to specify a single restart schedule in advance. From a managerial standpoint, this means that firms need not commit ex ante to one fixed memory length for pricing decisions. Instead, the adaptive procedure effectively learns whether the market calls for slower or faster forgetting and reallocates weight across restart schedules accordingly.

The dynamic regret bound in Theorem~\ref{the:uno_dynamic} can be refined to $ O(\mathrm{poly}(d) T^{7/9} V_T^{1/3}) $ by assuming $p=q=2$, which is standard in gradient estimation methods such as spherical smoothing,
ensuring that all dimension-dependent constants scale at most polynomially with $d$, and by appropriately setting  $\delta = C_{\delta}T^{-1/9}<1$ for some $C_{\delta}>0$. 
Moreover,
notice that the algorithm implicitly assumes that the time horizon $T$ is known in advance. This requirement can be removed via a standard technique known as the doubling trick, thereby yielding an anytime online algorithm  \citep{Lattimore2020}. The key idea is to maintain an initial guess of the time horizon and, whenever the actual number of iterations exceeds this guess, to double it and restart the algorithm. Specifically, the initial guess is set to $2$. Consequently, there are $K = \lfloor \log_2 (T+2) \rfloor$ epochs, where the $k$-th epoch consists of $T_k=2^{k}$ iterations for $k \in \{1,\ldots, K\}$. By directly applying H\"{o}lder's inequality, the regret guarantees that account solely for time dependence extend to this anytime version of the algorithm as follows:
\[
\sum_{k=1}^{K} T_k^{7/9} V_{T_k}^{1/3}
\leq
\left( \sum_{k=1}^{K} T_k^{7/6} \right)^{2/3}
\left( \sum_{k=1}^{K} V_{T_k} \right)^{1/3}
=
\left( \sum_{k=1}^{K} 2^{7k/6} \right)^{2/3}V_T^{1/3}
=
O(T^{7/9} V_T^{1/3}).
\]

\section{Numerical Experiments}\label{sec:num}

In this section, we conduct numerical experiments to illustrate the performance and practical application of the proposed pricing framework under nonstationary demand. The analysis proceeds in three parts. First, we conduct an ablation study to distinguish the roles of the three algorithmic components: the one-point mirror-ascent learner, the restarting mechanism, and the meta-learning procedure over restart schedules. This experiment highlights the role of each component in facilitating learning and adaptation. Second, we study synthetic pricing environments with different path-variation levels, ranging from nearly stationary markets to rapidly changing markets. These experiments show how the value of adaptive forgetting depends on the level of demand nonstationarity. Third, we construct a data-calibrated pricing environment using Walmart's real sales and price data to examine whether the proposed method delivers robust performance in a real-world retail setting. In all experiments, the seller receives only one revenue observation from the posted price in each period. 

\subsection{Ablation Study of the Three Algorithmic Components} \label{sec:ablation}

We first conduct an ablation study to isolate the roles of the three algorithmic components: the base one-point mirror ascent learner (Algorithm~\ref{alg:static}), the restarting mechanism (Algorithm~\ref{alg:dynamic}), and the expert-weighting meta-layer (Algorithm~\ref{alg:hedge}). To this end, we consider a two-dimensional nonstationary quadratic revenue environment. This specification captures the local behavior of a smooth revenue function around its optimal price and is consistent with the commonly used linear demand model in dynamic pricing. At each period $t$, the seller chooses a price vector $x=(x_1,x_2)^\mathtt{T}$ from the feasible set $\mathcal{X}=[-5,5]^2$. The expected demand for the two products is specified as $d_{1,t}(x) = b_{1,t} -  x_1/2, d_{2,t}(x) = b_{2,t} - x_2/2$, where $b_t=(b_{1,t},b_{2,t})^\mathtt{T}\in[-5,5]^2$ is a time-varying parameter that determines the period-specific optimal price. The corresponding expected revenue function is
$ r_t(x) = x_1 d_{1,t}(x)+x_2 d_{2,t}(x) = x_1(b_{1,t}- x_1/2)+x_2(b_{2,t}-x_2/2).
$
Equivalently,
$
r_t(x)=b_t^\mathtt{T} x- \|x\|^2/2.
$
This revenue function is concave in $x$, and its period-specific maximizer is $
x_t^*=b_t.$ The nonstationarity is induced through the time-varying sequence $\{b_t\}_{t=1}^T$. Performance is measured by the dynamic regret defined in \eqref{eq:d-reg}. 

We generate three levels of variation patterns, as shown in the left panels of Figures~\ref{fig:ab1}--\ref{fig:ab3}. 
We implement Algorithm~\ref{alg:hedge} with the following parameter configuration. The meta-learner pool consists of $N = 6$  experts with restart periods $\tau \in \{32, 64, 128, 256, 512, 1024\}$, forming a geometric sequence with ratio 2 that covers restart frequencies from sub-problem granularity up to the full horizon $T = 1000$. The $i$-th expert weight is initialized as $w_i \propto 1/(i(i+1))$. Each expert runs projected gradient ascent on the feasible box $\mathcal{X}$ with a fixed step size $\eta = 0.01$. At each round, the perturbation radius $\delta = 0.1$. The meta-learner updates expert weights via exponential weighting with rate $\varepsilon = 0.5$. Bandit feedback is corrupted by an additive Gaussian noise $\mathscr N(0, 0.1^2)$.
As a naive baseline, we also include a random policy that selects actions uniformly from $\mathcal{X}$. All results are averaged over 30 independent sample paths, and we report the mean cumulative dynamic regret together with $95\%$ confidence intervals. The dynamic regret results are shown in the right panels of Figures~\ref{fig:ab1}--\ref{fig:ab3}.

In Figure~\ref{fig:ab1}, the first case corresponds to a fixed $b_t$. In this setting, Algorithm~\ref{alg:static} matches the environment best and achieves the lowest regret, while Algorithm~\ref{alg:dynamic} and Algorithm~\ref{alg:hedge} incur larger regret due to the restarting mechanism. Moreover, for Algorithm~\ref{alg:dynamic}, more frequent restarts (i.e., smaller $\tau$) lead to worse performance. Figure~\ref{fig:ab2} presents the case with a low level of nonstationarity, while Figure~\ref{fig:ab3} corresponds to a higher level of nonstationarity. In both cases, the algorithms with restarting mechanisms perform better, showing their advantage in adapting to environmental changes.
Figure~\ref{fig:ab2} shows that the performance gap between the static and restarting methods begins to emerge, indicating that even a mild amount of temporal variation is sufficient to make adaptation beneficial. In Figure~\ref{fig:ab3}, this trend becomes more pronounced: the regret of the static method increases more rapidly, whereas restarting methods can track the changing maximizers more effectively. Overall, Figure~\ref{fig:ab2} and Figure~\ref{fig:ab3} together demonstrate that as the degree of nonstationarity increases, adaptive restarting becomes increasingly important for achieving low dynamic regret.

\begin{figure}
\centering
\includegraphics[width=\linewidth]{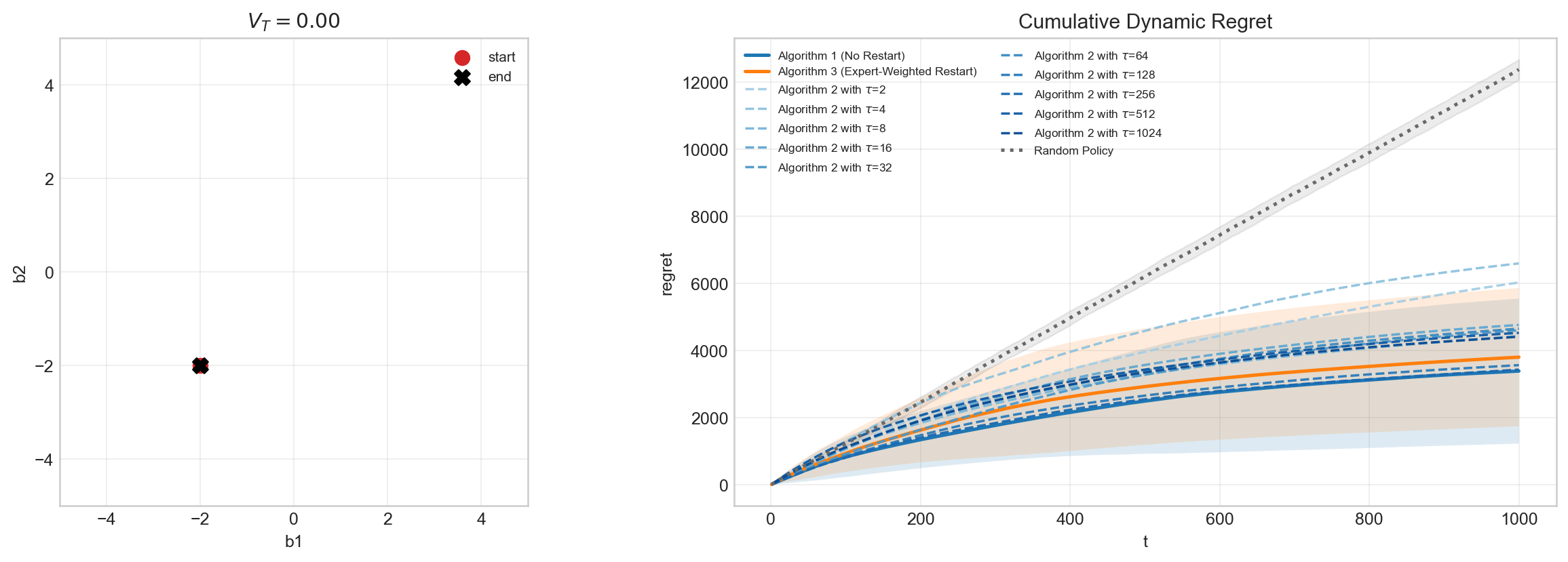}
\caption{No Variation}
\label{fig:ab1}
\end{figure}

\begin{figure}
\centering
\includegraphics[width=\linewidth]{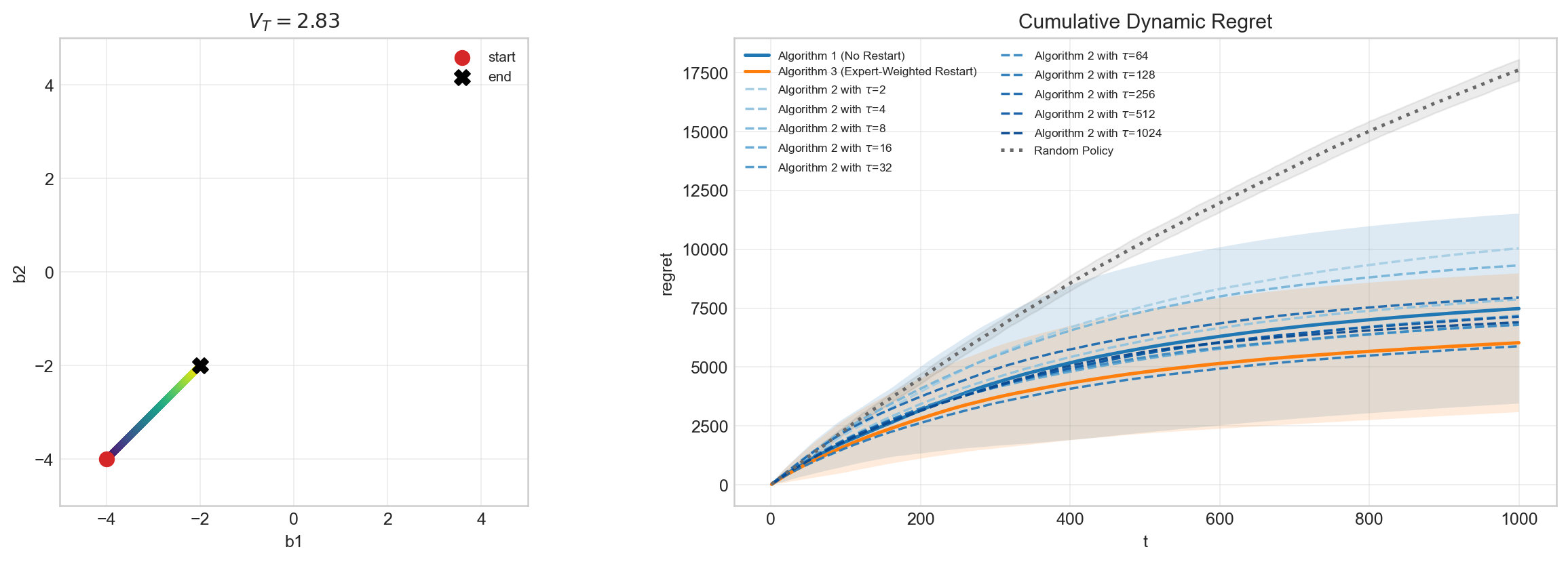}
\caption{Low Variation}
\label{fig:ab2}
\end{figure}

\begin{figure}
\centering
\includegraphics[width=\linewidth]{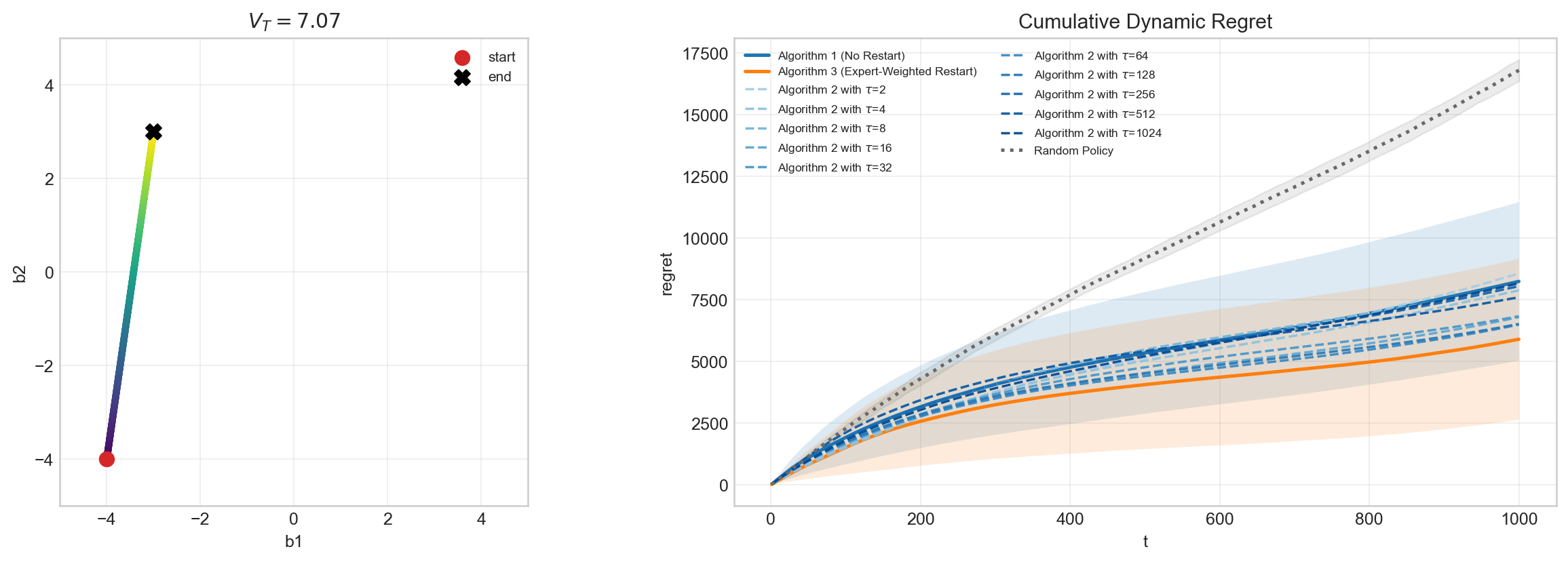}
\caption{High Variation}
\label{fig:ab3}
\end{figure}

\subsection{Synthetic Nonstationary Pricing Experiments with Controlled Variation}

To further systematically evaluate the impact of nonstationarity, we use the same linear demand as in Section~\ref{sec:ablation} but construct a family of $b$-paths with controlled variation budgets. The sequence $\{b_t\}_{t=1}^T$ is generated as a constrained random walk within the bounded domain $\mathcal{X} = [0,1]^2$. The process is initialized at $b_1 = c + \varepsilon$, where $c = (0.5, 0.5)^{\mathtt T}$ is the center of the domain and $\varepsilon$ is sampled uniformly from $[-0.3r, 0.3r]^2$ with $r = 0.5$. This ensures that the starting point is randomly perturbed while remaining well inside the domain. For each subsequent step $t = 2, \ldots, T$, we enforce a fixed parameter $\Delta_V:= V/(T-1)$, where $V$ denotes the target variation budget. At each round, a random direction $d_t$ is sampled uniformly from the unit sphere, and the next point is obtained as
\[
b_t=\max\left\{c-r,\;\min\left\{b_{t-1}+\Delta_V d_t,\;c+r\right\}\right\},
\]
where the maximum and minimum are applied componentwise, so as to guarantee feasibility within $\mathcal{X}$. By construction, the cumulative path variation satisfies
\[
\sum_{t=2}^{T} \|b_t - b_{t-1}\| \leq V.
\]
Due to boundary clipping, the realized variation may be slightly smaller than the target budget. In the experiments, we consider a range of variation levels $V \in \{0, 10, 20, 30, 40\}$, spanning regimes from nearly stationary to highly nonstationary.

We implement Algorithm~\ref{alg:hedge} with the same configuration as in Section~\ref{sec:ablation}. We compare our algorithm with vanilla UCB~\citep{ucb}, GP-UCB~\citep{gpucb}, EXP3~\citep{exp3}, HOO~\citep{hoo}, SW-UCB~\citep{Cheung2022MS}, BOB~\citep{Cheung2022MS}, and a naive random baseline. GP-UCB and HOO are continuous-space bandit algorithms and are therefore applied directly to the original feasible domain. By contrast, vanilla UCB, EXP3, SW-UCB, and BOB operate over a finite set of arms. For these discrete-action baselines, we discretize the continuous feasible region into a finite candidate set before running the algorithms. SW-UCB and BOB are nonstationary UCB-based algorithms for discrete action spaces, and thus serve as nonstationary discrete bandit baselines under the same discretization procedure. All benchmark algorithms are implemented in accordance with their respective original papers, and parameter configurations follow recommendations therein. For each method, we report four types of figures: the $b$-path, sample heatmaps, action trajectories, and dynamic regret. The $b$-path characterizes the nonstationarity of the environment. The sample heatmaps and action trajectories provide qualitative evidence of how well an algorithm tracks the evolving environment, while dynamic regret serves as the primary quantitative metric for fair comparison. Figures~\ref{fig:v10bpath}--\ref{fig:v10regret} show the case where $V=10$, and Figures~\ref{fig:v40bpath}--\ref{fig:v40regret} show the case where $V=40$. More results with different $V$ are shown in Section~\ref{sec:econnumerical} of the E-Companion.

The discrete UCB-based baselines exhibit relatively weak performance in our experiments. This is mainly because, after discretizing the continuous domain, these algorithms regard each grid point as a separate arm and need to explore each arm before effective exploitation. As the number of discretized arms increases, such arm-wise exploration becomes highly sample-inefficient, causing vanilla UCB, SW-UCB, and BOB to spend a substantial fraction of the budget on suboptimal candidate points. Under low variation, GP-UCB can perform relatively well, which is also consistent with the smooth and slowly varying structure of the $b$-path. However, as the variation level increases, our algorithm shows a clear advantage. From both the sample heatmaps and the action trajectories, we observe that our method is able to track the nonstationary environment more closely, allocating samples adaptively to the moving high-value regions. In contrast, the benchmark methods become less responsive under stronger nonstationarity, which leads to less efficient exploration and inferior tracking performance.

\begin{figure}
\centering
\includegraphics[width=.5\linewidth]{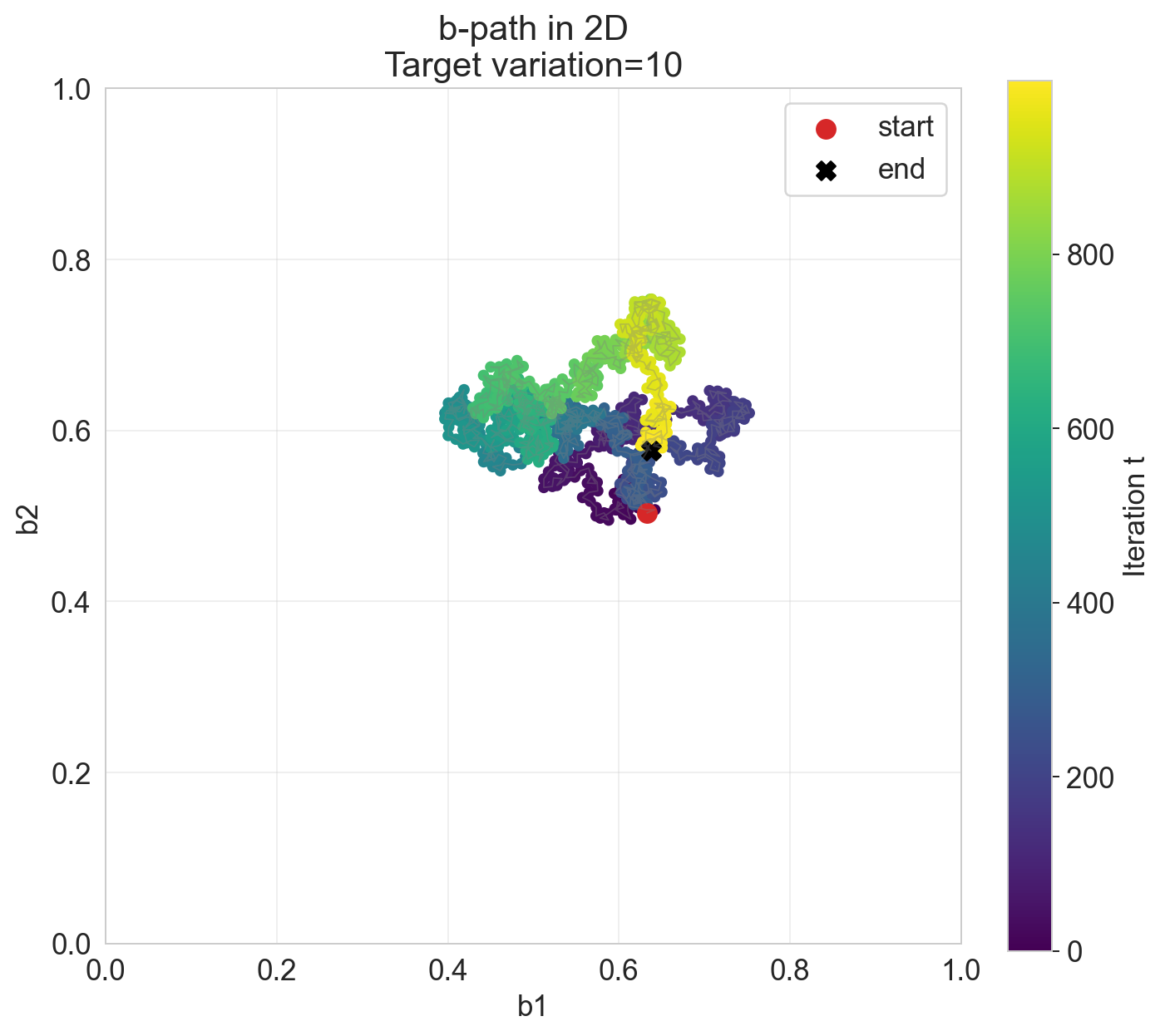}
\caption{Variation=10, b path}
\label{fig:v10bpath}
\end{figure}

\begin{figure}
\centering
\includegraphics[width=\linewidth]{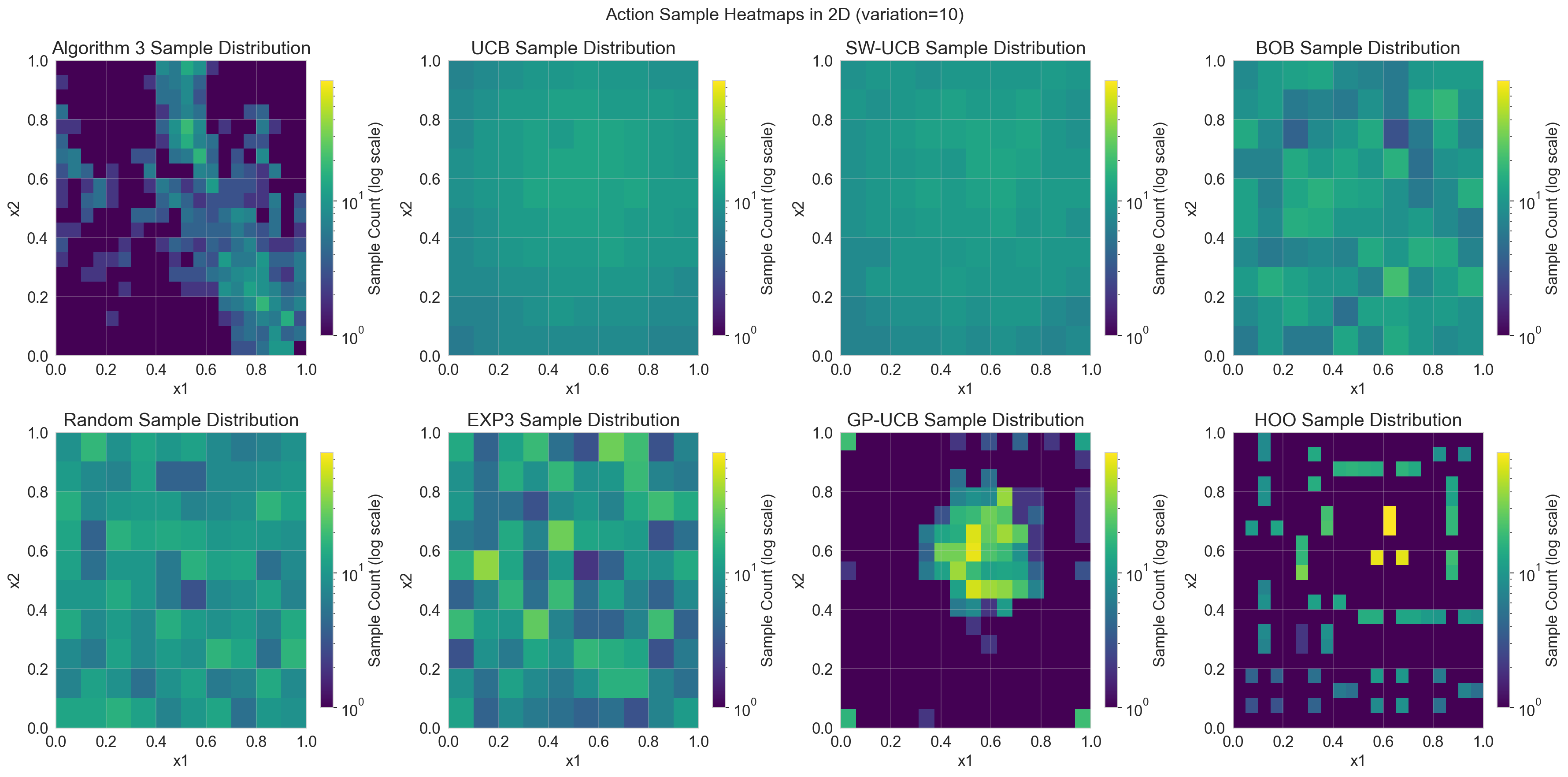}
\caption{Variation=10, sample heatmaps}
\label{fig:v10heatmap}
\end{figure}

\begin{figure}
\centering
\includegraphics[width=\linewidth]{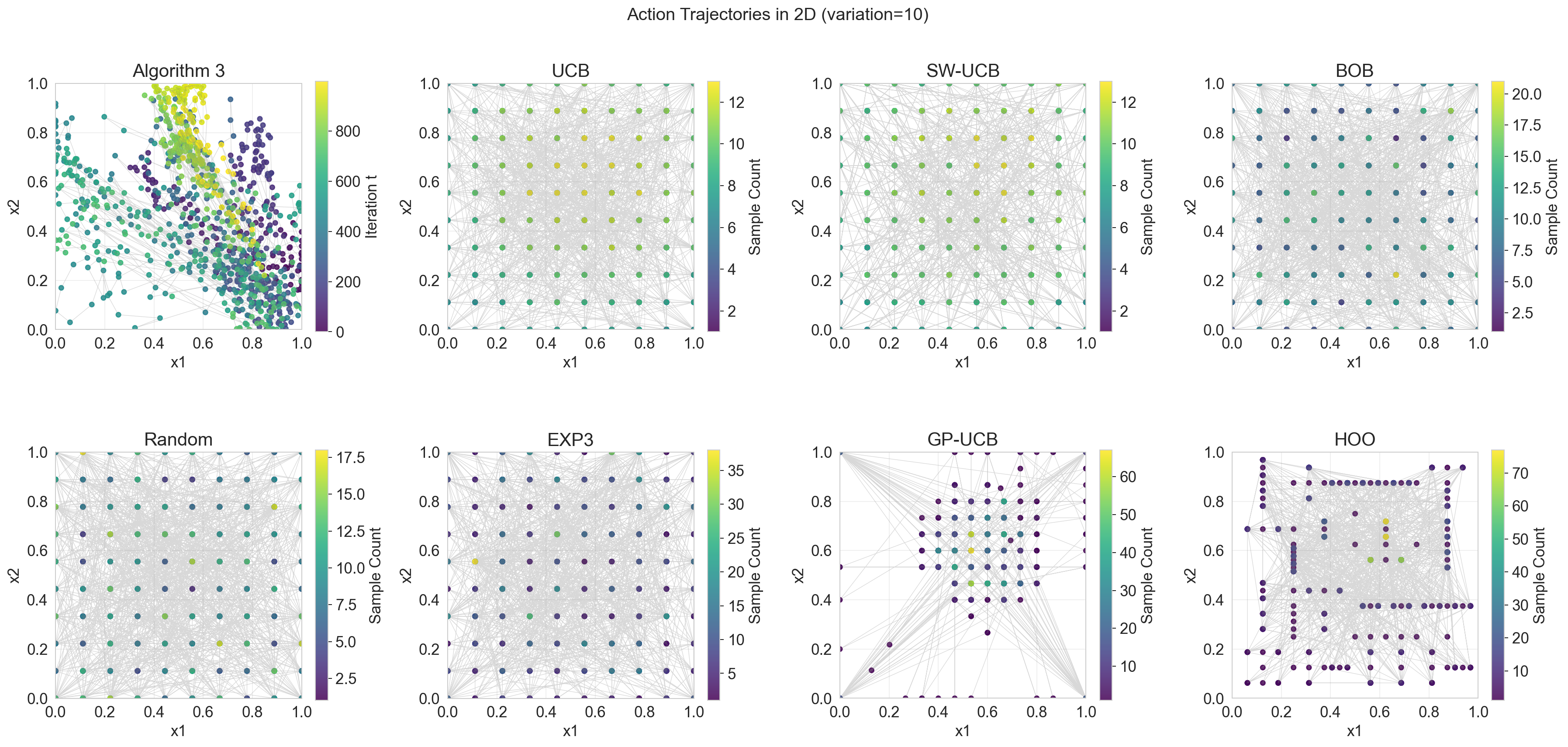}
\caption{Variation=10, action trajectories}
\label{fig:v10trajectory}
\end{figure}

\begin{figure}
\centering
\includegraphics[width=\linewidth]{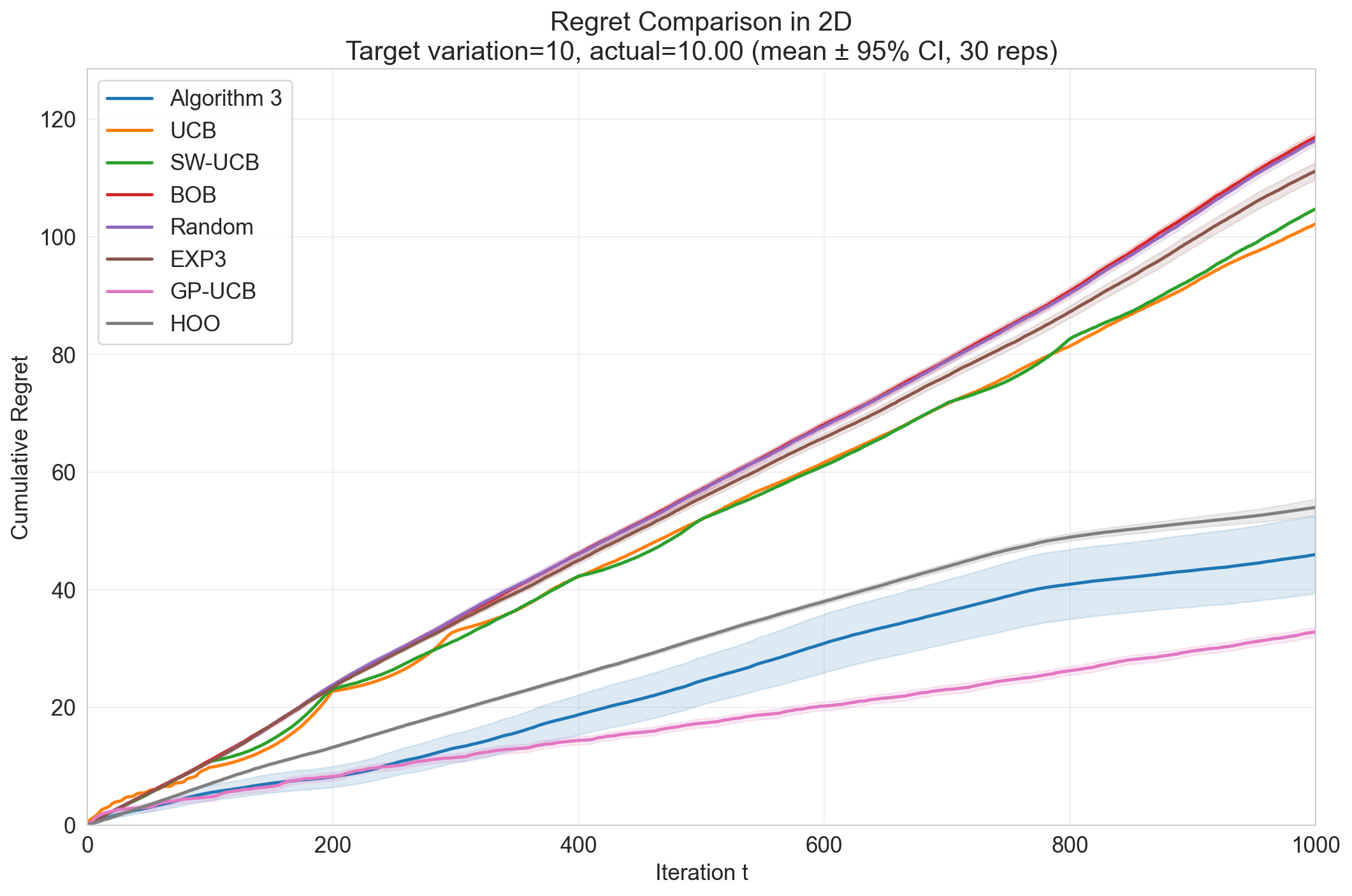}
\caption{Variation=10, regret }
\label{fig:v10regret}
\end{figure}

\begin{figure}
\centering
\includegraphics[width=.5\linewidth]{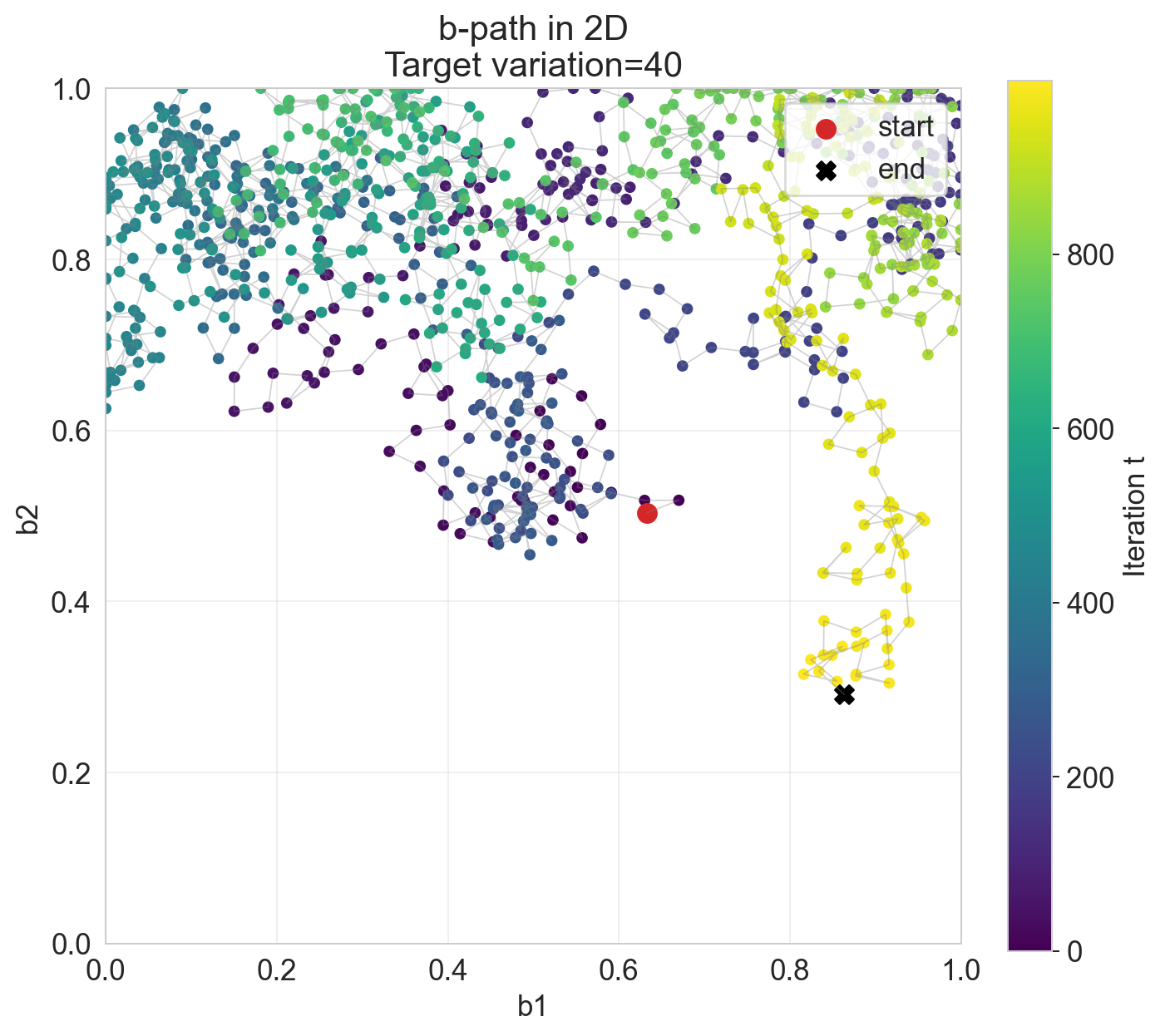}
\caption{Variation=40, b path}
\label{fig:v40bpath}
\end{figure}

\begin{figure}
\centering
\includegraphics[width=\linewidth]{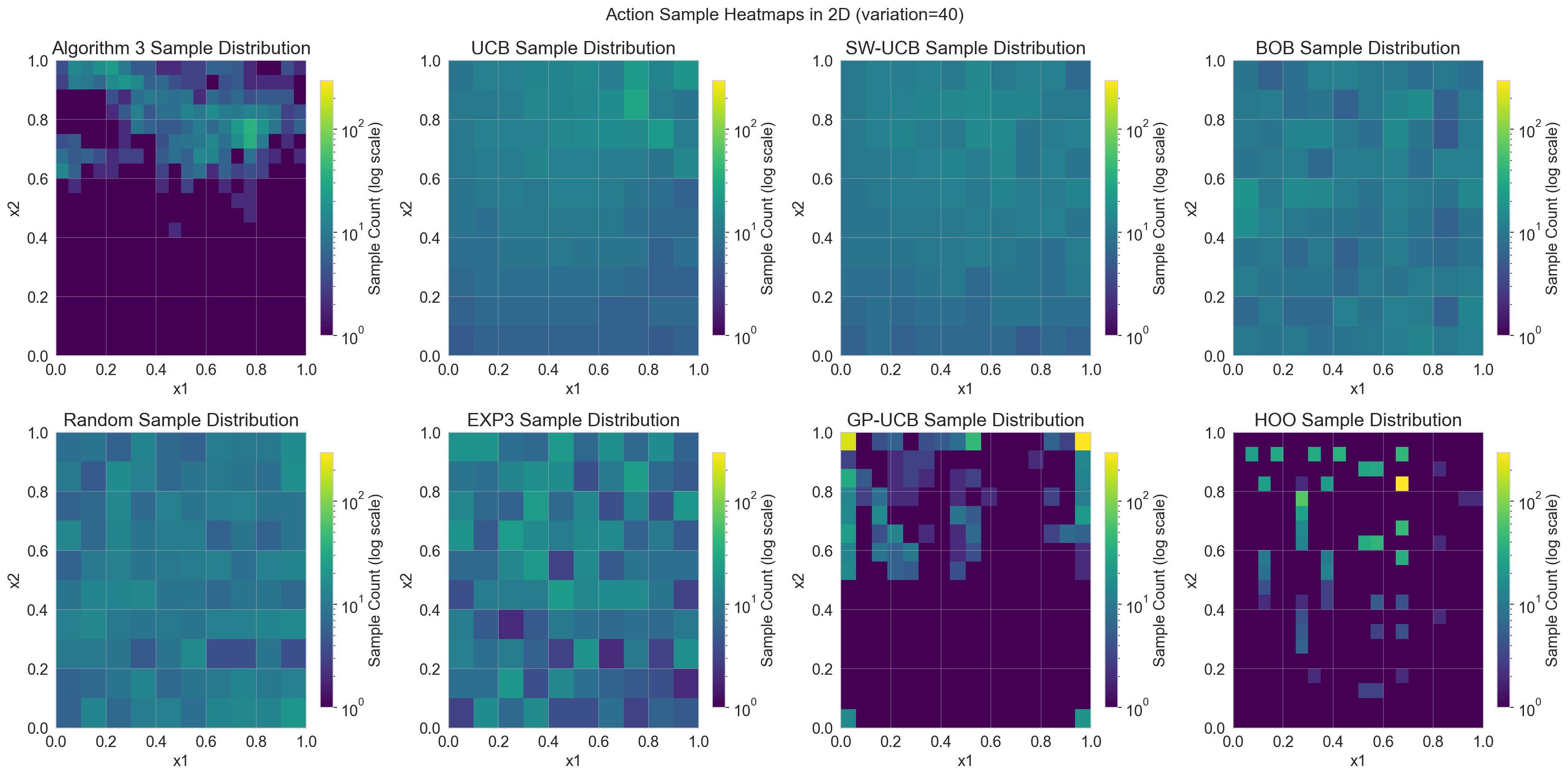}
\caption{Variation=40, sample heatmaps}
\label{fig:v40heatmap}
\end{figure}

\begin{figure}
\centering
\includegraphics[width=\linewidth]{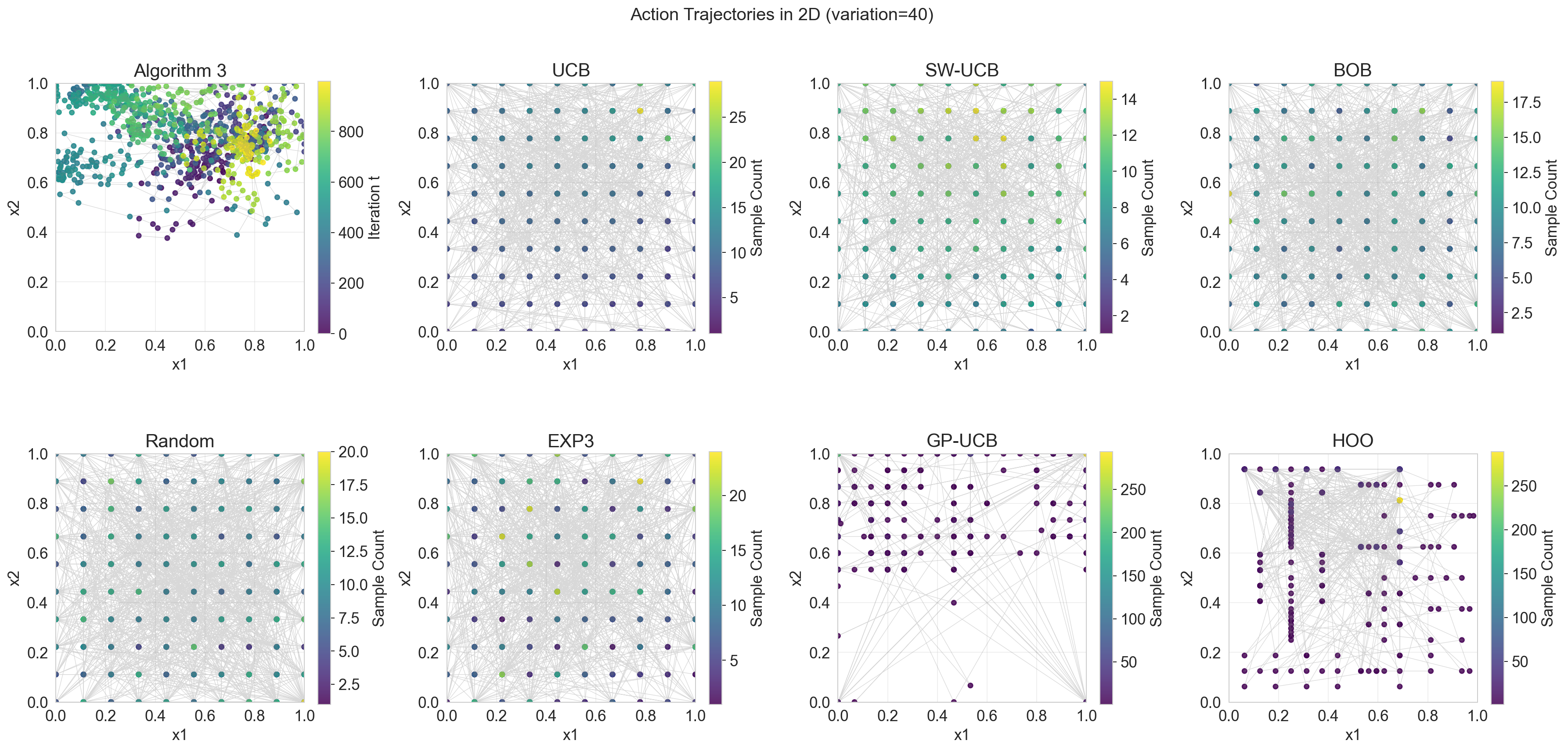}
\caption{Variation=40, action trajectories}
\label{fig:v40trajectory}
\end{figure}

\begin{figure}
\centering
\includegraphics[width=\linewidth]{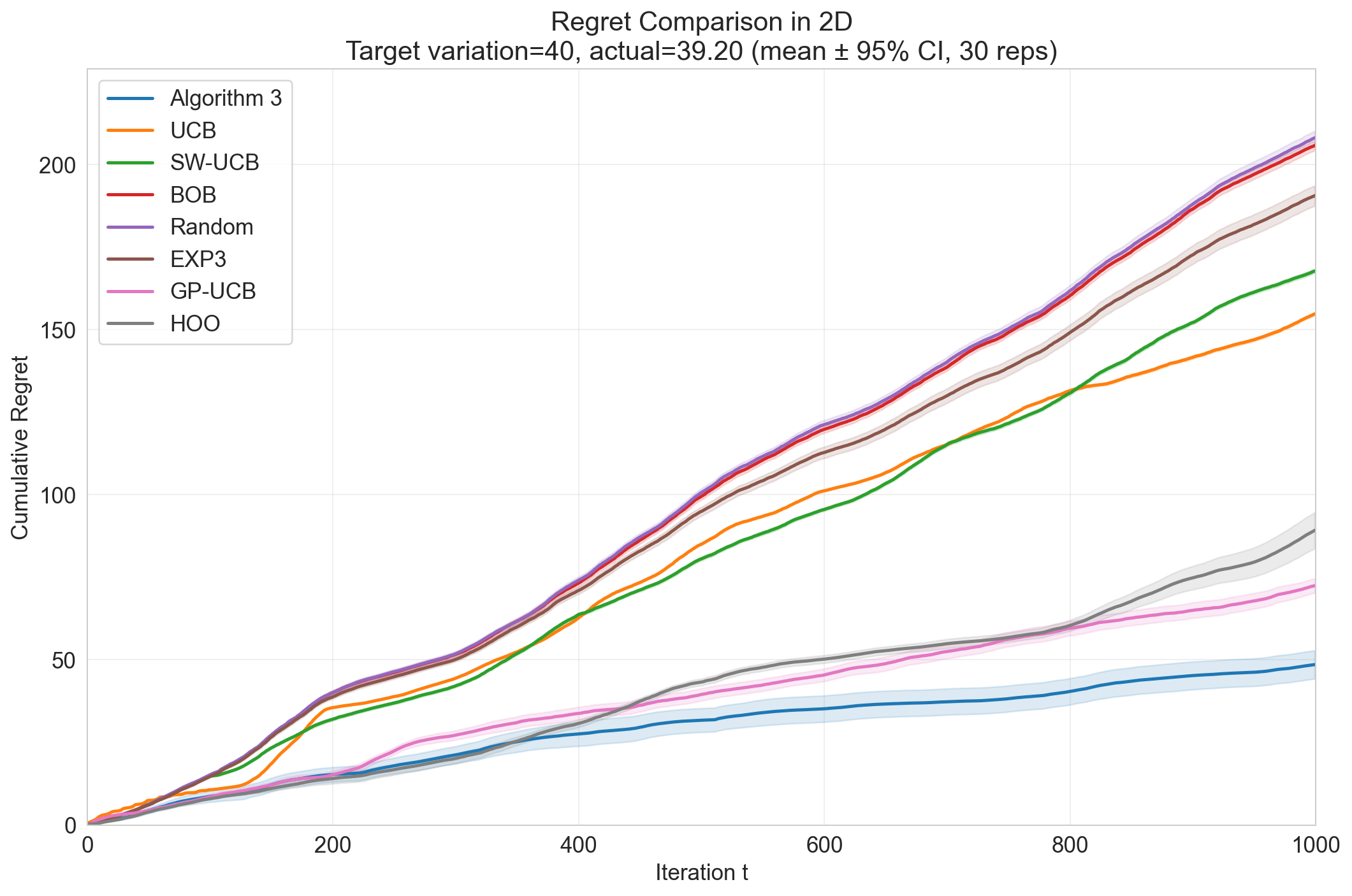}
\caption{Variation=40, regret }
\label{fig:v40regret}
\end{figure}

\subsection{Real-world Nonstationary Pricing Experiment Using Walmart Dataset}

To further evaluate the proposed policy in a more realistic nonstationary demand environment, we construct a data-calibrated retail pricing simulator using the Walmart dataset. While the preceding synthetic experiments allow us to control the degree and pattern of nonstationarity, they rely on stylized revenue functions and therefore may not fully capture the demand variability observed in retail operations. In practice, nonstationarity may arise from seasonality, product life cycles, local demand shocks, promotional effects, and changes in baseline market conditions. However, directly evaluating an online pricing policy in a real retail market is typically infeasible, since it would require experimental price interventions and repeated observations under alternative prices. We therefore build a simulator calibrated from real sales and price records: the historical data are used to construct item-level price--demand relationships and time-varying demand patterns, while the simulator allows different pricing policies to be evaluated under the same controlled nonstationary environment. This design preserves important empirical features of retail demand while maintaining a well-defined oracle benchmark for computing dynamic regret.

We use the Walmart M5 Forecasting Accuracy dataset, which contains daily sales records for multiple products across 10 Walmart stores in three U.S. states (CA, TX, and WI) from January 2011 to June 2016, spanning 1,913 days. The dataset includes daily sales, weekly prices, and a calendar file mapping dates to fiscal weeks. To ensure sufficient signal for demand estimation, we filter items based on average demand and retain only 54 items with consistently large numbers of observations. We then construct a weekly price--demand panel for each item. Specifically, daily sales are aggregated into weekly quantities and further summed across stores to obtain the total demand for each item-week pair. This yields a panel of observations.

For each item $i$ and period $t$, we estimate a time-varying price--demand curve from the filtered price--demand observations using polynomial regression. Depending on the number of distinct historical prices available for the item-period pair, we fit a quadratic, linear, or constant specification, yielding estimated coefficients $(\hat a_{i,t},\hat b_{i,t},\hat c_{i,t})$ from the period $t$ data. We then construct a stochastic pricing simulator over the 54 retained items. The feasible price set for each item is obtained by expanding its historical price range by $10\%$ on both ends. In period $t$, given a price vector $p_t=(p_{1,t},\ldots,p_{54,t})$, item-level demand is generated independently according to $ D_{i,t}\sim \mathrm{Poisson}(\lambda_{i,t}), $
where
$
\lambda_{i,t}=\max\{\hat a_{i,t}p_{i,t}^2+\hat b_{i,t}p_{i,t}+\hat c_{i,t},0\}.
$
The total revenue in period $t$ is
\[
R_t=\sum_{i=1}^{54}p_{i,t}D_{i,t}.
\]
Finally, we define an oracle benchmark by discretizing each item's feasible price interval into a grid of 201 points and selecting, in each period, the price vector that maximizes expected revenue under the current demand model. The cumulative oracle revenue serves as the benchmark for performance evaluation. This yields a data-calibrated nonstationary pricing environment with dimension $d=54$ and horizon $T=200$.

Since the total horizon $T=200$, we set $\tau \in \{16, 32, 64, 128, 256 \}$. Other configurations of Algorithm~\ref{alg:hedge} are kept the same as in Section~\ref{sec:ablation}. We compare Algorithm~\ref{alg:hedge} with HOO, a random policy, and the oracle benchmark. For the 54-dimensional problem, discrete-action baselines such as vanilla UCB, SW-UCB, and BOB require discretizing the continuous feasible region into a finite set of arms. However, in 54 dimensions, even a very coarse discretization would generate an exponentially large action set, making arm-wise exploration prohibitively inefficient. We therefore exclude these discrete baselines from this experiment. GP-UCB is also not included, since constructing and updating a Gaussian process model in a 54-dimensional space is computationally expensive and numerically challenging. 
HOO, by contrast, is a continuous-space bandit algorithm and can be applied without prior discretization, so it serves as the main bandit baseline in this high-dimensional setting. We report cumulative regret averaged across items and independent simulation replications. The regret is calculated relative to the oracle benchmark. We also plot the $95\%$ confidence intervals of HOO and our Algorithm~\ref{alg:hedge}. Figure~\ref{fig:readworld} shows that Algorithm~\ref{alg:hedge} consistently achieves the best performance among the compared methods. The results indicate that Algorithm~\ref{alg:hedge} is able to better adapt to the nonstationary pricing environment and exploit the evolving demand structure more effectively than the comparison methods. In particular, compared with HOO and the random baseline, it achieves higher cumulative revenue throughout the horizon and maintains a more stable advantage over repeated runs. This demonstrates the practical effectiveness of our method in high-dimensional real-world pricing problems.

\begin{figure}
\centering
\includegraphics[width=\linewidth]{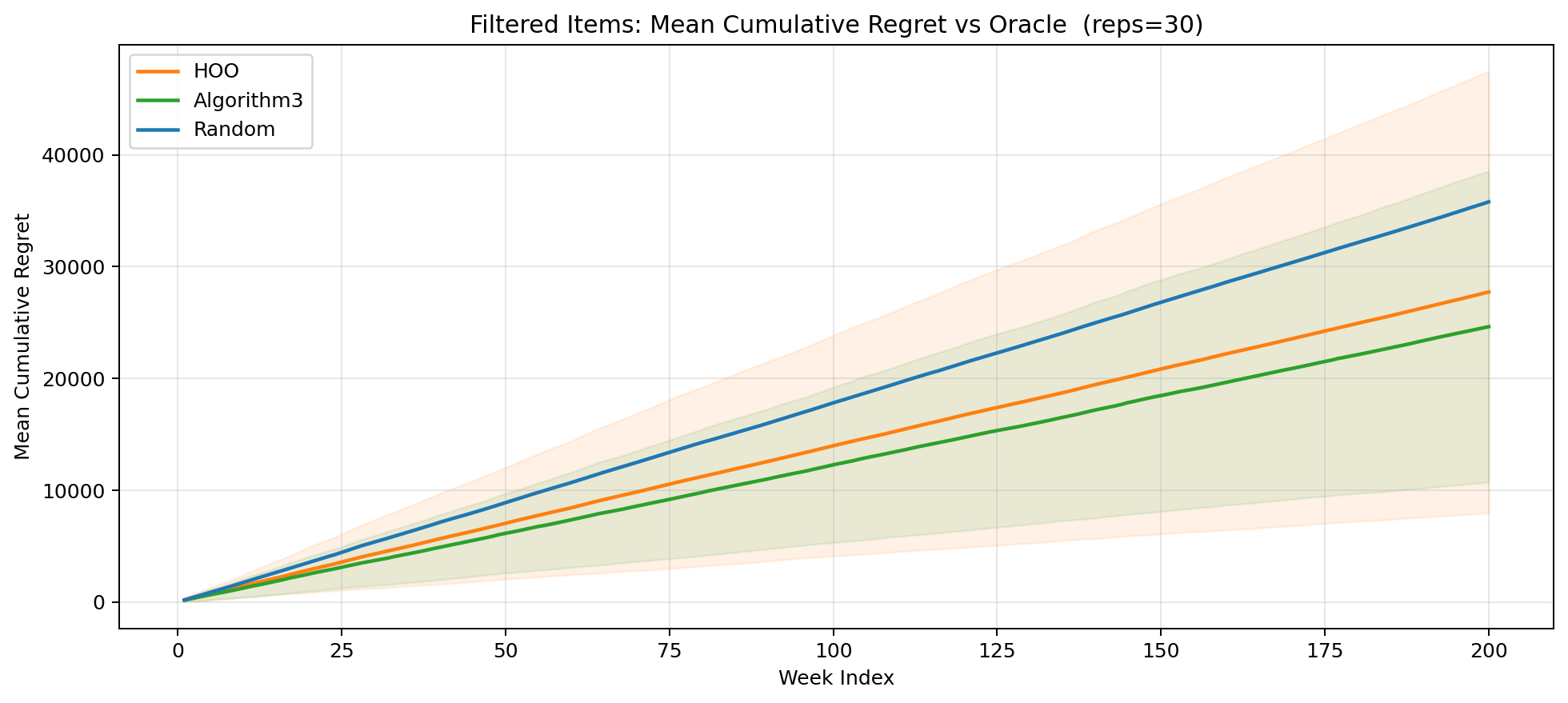}
\caption{Real-world nonstationary environment}
\label{fig:readworld}
\end{figure}

\section{Conclusion}\label{sec:conclu}

This paper studies a nonparametric dynamic pricing problem in a nonstationary environment with one-point feedback. The seller observes only realized revenues from a single posted price in each period, while the underlying revenue functions may change over time. We analyze this problem through a hierarchical construction that combines online mirror ascent with one-point gradient estimation, a restarting mechanism to address nonstationarity, and a meta-layer based on expert aggregation to accommodate unknown variation levels.
We first establish static regret guarantees for online mirror ascent under general bias--variance conditions on the gradient estimator. These results are then extended to the nonstationary stochastic setting via a batching and restarting procedure, leading to dynamic regret bounds that explicitly reflect the interplay between the time horizon and the path variation budget. When the variation budget is unknown, the bandit-over-bandit strategy allows the algorithm to hedge across multiple restarting scales while preserving comparable regret rates.
The theoretical analysis highlights how the attainable performance depends jointly on the smoothness of the revenue functions, the properties of the gradient estimator, and the degree of environmental variation. The numerical experiments further illustrate the behavior of the proposed procedures under different forms of nonstationarity and demonstrate their empirical effectiveness.

Two extensions are particularly worth exploring. The first is operations-management-oriented. It would be interesting to incorporate inventory constraints or capacity limits into nonstationary environments and study nonstationary joint pricing and inventory control problems \citep[see, e.g.,][]{Chen2021POM, Chen2021MOR}. Second, from a methodological perspective, it would be valuable to sharpen the dimension dependence of one-point feedback methods and establish information-theoretic minimax lower bounds under comparable variation measures.

\bibliographystyle{plainnat}

\let\oldbibliography\thebibliography
\renewcommand{\thebibliography}[1]{
\oldbibliography{#1}
\baselineskip14pt
\setlength{\itemsep}{10pt}
}

\bibliography{ref1}

\ECSwitch 

\ECHead{\centering E-Companion for \\
``Nonparametric Learning and Earning with One-Point Feedback under Nonstationarity''}

\section{Technical Supplements for Section~\ref{sec:known} }
\label{ec:sec3}

\proof{Proof of Proposition~\ref{pps:static}} 
Using Assumptions~\ref{asp:rf}--\ref{asp:oracle}, we have for every $x\in \Theta$,
\begin{align}\label{eceq:reg}
\sum_{t=1}^{\tau}r_t(x)- 
\mathsf E\left[\sum_{t=1}^{\tau}
r_t(x_t+\delta \tilde U_t)
\right] 
&\leq 
\sum_{t=1}^{\tau}r_t(x)- \mathsf E\left[
\sum_{t=1}^{\tau}\left(r_t(x_t)+\delta \nabla r_t(x_t)^{\mathtt T}\tilde U_t
-\frac{L}{2}\delta^2\|\tilde U_t\|^2\right)
\right]  \nonumber \\
&\leq \mathsf E\left[\sum_{t=1}^{\tau}r_t(x)- 
\sum_{t=1}^{\tau}r_t(x_t)\right]+\frac{\tau LC_u}{2}\delta^2 \nonumber\\
&\leq
\mathsf E\left[
\sum_{t=1}^{\tau} \nabla r_t(x_t)^{\mathtt T}(x-x_t)\right]+
\frac{\tau LC_u}{2}\delta^2 \nonumber\\
&=
\mathsf E\left[\sum_{t=1}^{\tau}
G_t^{\mathtt T}(x-x_t)\right]
+\sum_{t=1}^{\tau}\mathsf E\left[ (x-x_t)^{\mathtt T}
\mathsf E[\nabla r_t(x_t)-G_t|\mathcal F_t]
\right]+\frac{\tau LC_u}{2}\delta^2 \nonumber\\
&\leq
\mathsf E\left[\sum_{t=1}^{\tau}
G_t^{\mathtt T}(x-x_t)\right]
+C_1\delta^p \sum_{t=1}^{\tau}\mathsf E[\|x-x_t\|]
+\frac{\tau LC_u}{2}\delta^2 \nonumber \\
&\leq \mathsf E\left[\sum_{t=1}^{\tau}
G_t^{\mathtt T}(x-x_t)\right]
+\tau  C_1\delta^p\sqrt{d}
+\frac{\tau LC_u}{2}\delta^2,
\end{align}
where the first inequality follows from the smoothness of $r_t(\cdot)$, the second inequality holds because $\tilde U_t$ is conditionally mean zero, the third inequality is due to the concavity of $r_t(\cdot)$, the fourth inequality applies the bias assumption on $G_t$, and the last inequality relies on the previously stated condition $\mathrm{diam}(\Theta)\equiv \sqrt{d} $. Next, we bound the first term on the right-hand side of \eqref{eceq:reg}.

By the first-order optimality condition (or variational inequality) in the mirror ascent step (i.e., Step~8 of Algorithm~\ref{alg:static}), we have for all $x \in \Theta $,
\[
\left(\eta G_t-\nabla \psi(x_{t+1}) + \nabla \psi(x_t)
\right)^{\mathtt T}(x-x_{t+1})\leq 0.
\]
This is equivalent to
\begin{equation}\label{eceq:g1}
G_t^{\mathtt T}(x-x_{t+1})\leq
\frac{1}{\eta}\left(\nabla \psi(x_{t+1})-\nabla \psi(x_{t})\right)
^{\mathtt T}(x-x_{t+1})=\frac{1}{\eta}\left(
B_{\psi}(x,x_t)-B_{\psi}(x,x_{t+1})-B_{\psi}(x_{t+1},x_t)
\right),
\end{equation}
where the equality is from the well-known three-point identity for the Bregman divergence \citep[see, e.g.,][]{bibChen1993}. 
On the other hand, by Young's inequality and the $\alpha$-strong convexity of $\psi(\cdot)$, we obtain that
\begin{equation}\label{eceq:g2}
G_t^{\mathtt T}(x_{t+1}-x_{t})\leq 
\frac{\eta}{2\alpha}\|G_t\|^2
+
\frac{\alpha}{2\eta}\|x_{t+1} - x_t\|^2\leq 
\frac{\eta}{2\alpha}\|G_t\|^2
+\frac{1}{\eta}B_{\psi}(x_{t+1},x_t).
\end{equation}
Combining \eqref{eceq:g1} and \eqref{eceq:g2}  yields 
\begin{equation}\label{eceq:g3}
G_t^{\mathtt T}(x-x_{t})
=G_t^{\mathtt T}(x-x_{t+1})
+G_t^{\mathtt T}(x_{t+1}-x_{t})\leq 
\frac{1}{\eta}\left(
B_{\psi}(x,x_t)-B_{\psi}(x,x_{t+1})
\right)+\frac{\eta}{2\alpha}\|G_t\|^2.
\end{equation}
Summing the first term on the right-hand side of \eqref{eceq:g3} over $t$ gives
$$\sum_{t=1}^{\tau-1}\frac{1}{\eta}\left(
B_{\psi}(x,x_t)-B_{\psi}(x,x_{t+1})
\right)
=\frac{1}{\eta}\left(
B_{\psi}(x,x_1)-B_{\psi}(x,x_{\tau})
\right)
\leq \frac{B}{\eta}.$$
Also,  we use Assumption~\ref{asp:oracle} to bound the expectation of the second term on the right-hand side of \eqref{eceq:g3}:
\begin{align*}
\mathsf E[\| G_t\|^2]&=
\mathsf E\left[\mathsf E[\| G_t\|^2|\mathcal F_t]\right]  \\
&=\mathsf E\left[
\mathsf E[ \|G_t-\mathsf E[G_t|\mathcal F_t]\|^2 |\mathcal F_t]
+  \| \mathsf E[G_t|\mathcal F_t] \|^2 
\right]\\
&\leq
C_2\delta^{-q}+
2\mathsf E\left[
\| \mathsf E[G_t|\mathcal F_t] -\nabla r_t(x_t)\|^2
\right]
+2\mathsf E\left[
\|\nabla r_t(x_t)\|^2
\right]\\
&\leq
C_2\delta^{-q}+2C_1^2\delta^{2p}+2C_r.
\end{align*}
Combining these results with  \eqref{eceq:reg} yields
\begin{equation*}
\sum_{t=1}^{\tau}r_t(x)- 
\mathsf E\left[\sum_{t=1}^{\tau}
r_t(x_t+\delta \tilde U_t)
\right] \leq
\frac{B}{\eta}+\frac{\tau\eta}{\alpha}
\left(C_1^2\delta^{2p}+\frac{C_2}{2}\delta^{-q} +C_r \right)
+\tau  C_1\delta^p\sqrt{d}
+\frac{\tau LC_u}{2}\delta^2.
\end{equation*}
Since the above holds for all $x\in \Theta$, the desired result follows.
\Halmos  
\endproof

\proof{Proof of Theorem~\ref{the:dyn_reg}}

We denote by $\{\mathscr T_1,\ldots,\mathscr T_{\lceil T/\tau \rceil}\}$ the collection of partitioned batches, each of size $\tau$ except possibly the last one. It follows that the regret admits the decomposition  $$\sum_{t=1}^{T}r_t(x_t^*)- \mathsf E\left[\sum_{t=1}^{T}r_t(\tilde x_t)		\right] =\sum_{b=1}^{\lceil T/\tau \rceil}R_b,$$ where
\begin{equation*}
R_b:=
\sum_{t \in \mathscr T_b} r_t(x^*_t) - \mathsf E\left[ \sum_{t \in \mathscr T_b} r_t(\tilde x_t)
\right]=\underbrace{\max_{x\in\Theta}\sum_{t \in \mathscr T_b} 
	r_t(x)-\mathsf E\left[ \sum_{t \in \mathscr T_b}  r_t(\tilde x_t) \right]}
_{J_{1,b}}
+\underbrace{\sum_{t \in \mathscr T_b} r_t(x^*_t)-
	\max_{x\in\Theta}\sum_{t \in \mathscr T_b}r_t(x)}_{J_{2,b} }.
	\end{equation*}
	Note that the first term, $J_{1,b}$, is exactly the static regret incurred in batch $b$ under the adversarial OCO setting, whereas the second term, $J_{2,b}$, quantifies the performance gap over batch $b$ between the best single action in hindsight and the dynamic benchmark. We next derive upper bounds for these two terms.
	
	On the one hand, since the upper bound is nondecreasing in the batch length and each batch contains at most $\tau$ periods, we obtain
	$
	J_{1,b}\leq U_b,
	$
	where  $U_b$ is the upper bound on the right-hand side of \eqref{eq:oco_reg}, with $\delta$ and $\eta$ replaced by $\delta_b$ and $\eta_b$, respectively.
	On the other hand,  let $V_b=\max_{ \{x_t^*\in \Theta_t^*\}_{t\in \mathscr T_{b}} }\sum_{t \in \mathscr T_{b} }\|x_t^*-x_{t-1}^*\|$, so that the total variation budget constraint can be equivalently written as
	$
	\sum_{b=1}^{\lceil T/\tau \rceil}V_b
	\leq V_T.
	$
	Let $\bar t$ denote an arbitrary period in batch $b$, and recall that $L_r$ is the Lipschitz constant of $r_t(\cdot)$. Then
	\begin{equation*}
J_{2,b}
\leq 
\sum_{t \in \mathscr T_b}\left(
r_t(x^*_t) - r_t(x^*_{\bar t})
\right)
\leq
L_r \sum_{t \in \mathscr T_b} \|x^*_t-x^*_{\bar t}\| 
\leq L_r\tau V_b.
\end{equation*}

Summing over all batches, we have
\[
\sum_{t=1}^{T}r_t(x_t^*)- \mathsf E\left[\sum_{t=1}^{T}r_t(\tilde x_t)		\right]
= \sum_{b=1}^{\lceil T/\tau \rceil} (J_{1,b} + J_{2,b})
\le \sum_{b=1}^{\lceil T/\tau \rceil} U_b
+L_r \tau V_T .
\]
This completes the proof.
\Halmos  
\endproof

\proof{Proof of Corollary~\ref{crl:dynamic}} 

Based on the results in Corollary~\ref{crl:static} and Theorem~\ref{the:dyn_reg}, the dynamic regret satisfies 
\begin{align}\label{eceq:dr1}
\sum_{t=1}^{T}r_t(x_t^*)- \mathsf E\left[\sum_{t=1}^{T}r_t(\tilde x_t)		\right] 
&\leq \sum_{b=1}^{\lceil T/\tau \rceil}U_{b}
+L_r\tau V_T\nonumber \\ 
&\leq \left\lceil  \frac{T}{\tau}\right\rceil O\left(  \tau_b^{ \frac{\hat p +q}{2\hat p +q} } \right)  + L_r\tau V_T\nonumber\\
&\leq \left( \frac{T}{\tau} +1\right)O\left(  \tau^{ \frac{\hat p +q}{2\hat p +q} } \right)+ L_r\tau V_T \nonumber \\
&\leq 2\frac{T}{\tau}O\left(  \tau^{ \frac{\hat p +q}{2\hat p +q} } \right)+ L_r\tau V_T \nonumber \\
&=O\left(T  \tau^{ -\frac{\hat p}{2\hat p +q} } \right)+ L_r\tau V_T
\end{align}
where the third inequality uses the fact that $\tau_b\leq \tau$, while the fourth inequality follows from $\tau \leq T$, since $\tau$ is chosen as $\tau=\lceil (\sqrt{d}T/V_T)^{(2\hat p +q )/(3\hat p +q)} \rceil$ and we recall that $\sqrt{d} \leq V_T \leq \sqrt{d} T$.
Then,  we have 
\begin{align*}
\eqref{eceq:dr1}&\leq O\left( T \left(\frac{\sqrt{d}T}{V_T}\right)^{-\frac{\hat p}{3\hat p +q}  }\right)
+L_r\left( \left(\frac{\sqrt{d}T}{V_T}\right)^{ \frac{2\hat p+q}{3\hat p +q}  }+1\right)V_T\\
&\leq O\left(  d^{-\frac{\hat p}{6\hat p +2q}} T^{ \frac{2\hat p + q}{3\hat p +q} } V_T^{  \frac{\hat p}{3\hat p +q} } \right)
+2L_r d^{\frac{2\hat p+q}{6\hat p +2q}}T^{ \frac{2\hat p + q}{3\hat p +q} } V_T^{  \frac{\hat p}{3\hat p +q} }\\
&=O\left(\mathrm{poly}(d)
T^{ \frac{2\hat p + q}{3\hat p +q} } V_T^{  \frac{\hat p}{3\hat p +q} }\right).
\end{align*}
\Halmos  
\endproof

\section{Technical Supplements for Section~\ref{sec:unknown} }
\label{ec:sec4}

\proof{Proof of Theorem~\ref{the:uno_dynamic}} 
Based on Assumptions~\ref{asp:rf}--\ref{asp:oracle}, we first observe that it is possible to construct a new sequence of functions $\{\bar r_t(\cdot)\}_{t=1}^T$ such that the resulting gradient estimator $G_t$ is conditionally unbiased with respect to $\nabla \bar r_t(\cdot)$. Specifically, for any $y \in \Theta$, define
\[
\bar r_t(y) = \mathsf{E}\Big[ r_t(y) + (G_t - \nabla r_t(x_t))^\mathtt{T} y \big| \mathcal{F}_t \Big],
\]
which implies
$
\nabla \bar r_t(y) = \nabla r_t(y) - \nabla r_t(x_t) + \mathsf{E}[G_t \mid \mathcal{F}_t].
$
Thus, evaluating $\nabla \bar r_t(y)$ at $y = x_t$ gives
\[
\nabla \bar r_t(x_t) = \mathsf{E}[G_t \mid \mathcal{F}_t].
\]
Moreover, since $\|\mathsf E\left[ \nabla r_t(x_t)-G_t | \mathcal F_t \right] \| 
\leq C_1 \delta^p$ (Assumption~\ref{asp:oracle}), it follows that for all $y \in \Theta$,
\begin{equation}\label{eceq:rr}
|\bar r_t(y) -  r_t(y) |
=|y^{\mathtt T}\mathsf E[G_t-\nabla r_t(x_t)|\mathcal F_t]|
\leq C_1\delta^p\sqrt{d},
\end{equation}
where we recall that the decision space has been normalized so that $\sup_{y\in \Theta}\|y\|=\sqrt{d}$. 
On the other hand, by the concavity of $r_t(\cdot)$, for all $y \in \Theta$ we have
\begin{align*}
\bar r_t(y)-\bar r_t(x_t)
&=
r_t(y)-r_t(x_t)
+\mathsf E\left[ 
(y-x_t)^{\mathtt T}(G_t-\nabla r_t(x_t))
|\mathcal F_t \right]\\
&\leq
(y-x_t)^{\mathtt T}\nabla r_t(x_t) + 
(y-x_t)^{\mathtt T}\left(\nabla \bar r_t(x_t)-\nabla r_t(x_t)\right)\\
&=
(y-x_t)^{\mathtt T}\nabla \bar r_t(x_t).
\end{align*}
By setting $y = x^*_t$ and taking the expectation, we obtain 
\begin{equation}\label{eceq:cav_bar}
\mathsf E[\bar r_t(x^*_t)-\bar r_t(x_t)]\leq
\mathsf E\left[(x^*_t-x_t)^{\mathtt T}\nabla \bar r_t(x_t)\right]=
\mathsf E\left[(x^*_t-x_t)^{\mathtt T} 
\mathsf E\left[G_t |\mathcal F_t \right]
\right]
=\mathsf E\left[(x^*_t-x_t)^{\mathtt T} G_t\right].
\end{equation}
Hence, combining \eqref{eceq:rr} and \eqref{eceq:cav_bar} above, we get that for every expert $i \in \{1, \ldots, N\}$,
\begin{align}\label{eceq:brace}
\mathsf E\left[\sum_{t=1}^T
\left(	r_t(x^*_t)-r_t(x_t)\right)
\right]&\leq 
\mathsf E\left[
\sum_{t=1}^T \left( \bar r_t(x^*_t)-\bar r_t(x_t) \right)
\right]+2 C_1 \delta^p\sqrt{d}T \nonumber \\
&\leq 
\mathsf E\left[
\sum_{t=1}^T G_t^{\mathtt T} (x^*_t-x_t)
\right]+2 C_1 \delta^p\sqrt{d}T \nonumber \\
&=\underbrace{\mathsf E\left[
	\sum_{t=1}^T G_t^{\mathtt T} \left(x^*_t-x_t^{(i)}\right)
	\right]}_{\texttt{i-th expert regret} }
+\underbrace{\mathsf E\left[
	\sum_{t=1}^T G_t^{\mathtt T} \left(x_t^{(i)}-x_t \right)
	\right]}_{\texttt{meta regret} }
+2C_1 \delta^p\sqrt{d}T.
\end{align}

Next, we bound the first two terms on the right-hand side of \eqref{eceq:brace} separately.
We begin by analyzing the expert regret. Since the bound holds uniformly over all experts, we focus on the ``best” expert, denoted by $i^*$, whose precise definition is given in \eqref{eceq:i_star} below. On the one hand,
by an argument analogous to that used in the static setting (see Section~\ref{ec:sec3}), within each batch we obtain that for any $x \in \Theta$,
\[
\mathsf{E}\left[\sum_{t=1}^{\tau^{(i^*)}} G_t^\mathtt{T}\left(x - x_t^{(i^*)}\right) \right]
\leq \frac{B}{\eta^{(i^*)}} + \frac{\tau^{(i^*)}\eta^{(i^*)}}{2\alpha} G_{\delta}^2,
\]
where $G_{\delta}:=\sqrt{2C_1^2\delta^{2p} +C_2\delta^{-q}+2C_r }$ for notational convenience.
Choosing $\eta^{(i^*)}=\sqrt{2\alpha B/(G^2_{\delta}\tau^{(i^*)})}$ by the AM--GM inequality and taking the maximum over $x$, we have
\begin{equation}\label{eceq:in}
\max_{x\in \Theta}\sum_{t=1}^{\tau^{(i^*)} }
\mathsf E\left[G_t^{\mathtt T}x\right]
- \mathsf E
\left[\sum_{t=1}^{\tau^{(i^*)} }G_t^{\mathtt T}x_t^{(i^*)}\right]
\leq 
\sqrt{\frac{2B}{\alpha} } G_{\delta} \sqrt{\tau^{(i^*)} }.
\end{equation}
On the other hand, to complete the analysis of the expert regret, it remains to bound the term
$$
\mathsf E\left[\sum_{t \in \mathscr T_b}G_t^{\mathtt T} x^*_t\right]-
\max_{x\in\Theta}\sum_{t \in \mathscr T_b}
\mathsf E\left[G_t^{\mathtt T}x\right],
$$
where, with a slight abuse of notation, $\mathscr T_b$ denotes all time periods in the $b$-th batch,  with $b \in \{1,\ldots, \lceil T/\tau^{(i^*)} \rceil \}$. 
To this end, again, denote by $\bar t$ an arbitrary period in  $\mathscr T_b$. Recalling that $V_b$ denotes the path variation budget over the $b$-th batch and that $C_r = \sup_{x \in \Theta,\,1\leq t\leq T} \| \nabla r_t(x) \|^2$, we have
\begin{align}\label{eceq:out}
\mathsf E\left[\sum_{t \in \mathscr T_b}G_t^{\mathtt T} x^*_t\right]-
\max_{x\in\Theta}\sum_{t \in \mathscr T_b}
\mathsf E\left[G_t^{\mathtt T}x\right]
&\leq \mathsf E\left[\sum_{t \in \mathscr T_b}
G_t^{\mathtt T} (x^*_t-x^*_{\bar t})\right] \nonumber \\
&\leq
\sum_{t \in \mathscr T_b}\mathsf E\left[
\|x^*_t-x^*_{\bar t}\| \| \mathsf E[G_t|\mathcal F_t] \|
\right]  \nonumber \\
&\leq
\sum_{t \in \mathscr T_b}V_b\mathsf E\left[
\|\mathsf E[G_t|\mathcal F_t]-\nabla r_t(x_t)+ \nabla r_t(x_t)\|
\right]  \nonumber \\
&\leq
\sum_{t \in \mathscr T_b}V_b\left( C_1\delta^p+\sqrt{C_r}\right)  \nonumber\\
&\leq
\tau^{(i^*)}V_b\left( C_1\delta^p+\sqrt{C_r}\right).
\end{align}
Combining  inequalities \eqref{eceq:in} and \eqref{eceq:out}, we derive that
\begin{align}\label{eceq:exp_reg}
\mathsf E\left[
\sum_{t=1}^T G_t^{\mathtt T} \left(x^*_t-x_t^{(i^*)}\right)
\right]
&\leq
\sum_{b=1}^{\lceil T/\tau^{(i^*)} \rceil}
\left(
\sqrt{\frac{2B}{\alpha} } G_{\delta} \sqrt{\tau^{(i^*)} }+
\tau^{(i^*)}V_b\left( C_1\delta^p+\sqrt{C_r}\right)
\right) \nonumber \\
&= 
\sqrt{\frac{2B}{\alpha} }\left\lceil T/\tau^{(i^*)}\right\rceil
G_{\delta} \sqrt{\tau^{(i^*)} }
+\left( C_1\delta^p+\sqrt{C_r}\right)
\tau^{(i^*)}V_T \nonumber \\
&\leq
2\sqrt{\frac{2B}{\alpha} }\frac{T}{\sqrt{\tau^{(i^*)} }}G_{\delta}
+\left( C_1\delta^p+\sqrt{C_r}\right)
\tau^{(i^*)}V_T,
\end{align}
where in the last step we implicitly utilize the condition $T \geq \tau^{(i^*)}$, which will be justified later.
To ensure that the right-hand side of \eqref{eceq:exp_reg} is small, we aim to choose $\tau^{(i^*)}$  close to $\lceil (\sqrt{d}T / V_T)^{2/3} \rceil$, which is not directly available because $V_T$ is unknown. To address this difficulty, note that $\sqrt{d} \leq V_T \leq \sqrt{d} T$, implying
$$
1 \leq \lceil (\sqrt{d}T/V_T)^{2/3} \rceil  \leq \left\lceil T
^{2/3} \right\rceil.
$$
Hence, we construct a pool of candidate batch sizes to discretize the above range:
\[
\left\{ \tau^{(i)} = 2^{i-1}, \, i = 1, \ldots, N \right\},
\]
where $N$ is chosen to satisfy
\[
\log_2 \left\lceil T^{2/3} \right\rceil + 1 \leq N \leq \log_2 T+1.
\]
The above lower bound ensures that the range of $\lceil (\sqrt{d}T / V_T)^{2/3} \rceil$ is fully covered by the exponential grid, and the upper bound guarantees that the largest $\tau^{(i)}$ does not exceed $T$. Since the condition $T\geq 14$ is sufficient to guarantee that the difference between the upper and lower bounds exceeds $1$,  we can set $N=\lceil\log_2  \lceil T^{2/3} \rceil \rceil+ 1.$
Notice that there are only $O(\log_2 T)$ candidate batch sizes, which do not incur too much computational overhead. 
It follows that there exists a unique $i^*$ such that
\begin{equation}\label{eceq:i_star}
\left\lceil \left(\frac{\sqrt{d}T}{V_T}\right)
^{\frac{2}{3} } \right\rceil
\in [\tau^{(i^*)}, \tau^{(i^*+1)}]=[\tau^{(i^*)}, 2\tau^{(i^*)}].
\end{equation}
Thus, it holds that
\begin{align}\label{eceq:exp_reg_bound}
\eqref{eceq:exp_reg}
&\leq
2\sqrt{\frac{2B}{\alpha} }
\frac{T}{\sqrt{\frac{1}{2}\lceil (\sqrt{d}T/V_T)^{2/3} \rceil }}G_{\delta}
+\left( C_1\delta^p+\sqrt{C_r}\right)
\lceil (\sqrt{d}T/V_T)^{2/3} \rceil V_T \nonumber \\
&\leq
\left(4d^{-1/3}\sqrt{\frac{B}{\alpha} }G_{\delta}
+2d^{1/3}\left( C_1\delta^p+\sqrt{C_r}\right)\right)
T^{2/3}V_T^{1/3}.
\end{align}

Second, we analyze the meta regret, i.e., the second term on the right-hand side of \eqref{eceq:brace}. To this end, we adapt the standard analysis tools for the exponentially weighted average forecaster with nonuniform initial weights \citep{bibCB2006}. Recall that $\mathscr F_t=\sigma\{\{x_t^{(i)},i=1,\ldots,N\}, \tilde U_1,\bar U_1, \xi_1, \ldots,  \tilde U_{t-1}, \bar U_{t-1}, \xi_{t-1}, \tilde U_{t}, \bar U_{t}\}$, representing the information immediately prior to observing the revenue. Note that $\{x_t, x_t^{(i)}, \, i = 1, \ldots, N\}$ are measurable with respect to $\mathscr F_t$.
Define the time-varying loss function $\ell_t(y):=-G_t^{\mathtt T}y $ for all $y \in \Theta$. 
For notational convenience,
let $\ell_t^{(i)}=\ell_t(x_t^{(i)}) - \ell_t(x_t)$ and $L_t^{(i)}=\sum_{s=1}^t\ell_s\left(x_s^{(i)}\right)$ for all $i=1,\ldots,N$. Define the potential function
$W_t
=\sum_{i=1}^N w_1^{(i)}\exp(-\varepsilon L_t^{(i)})$.

The meta regret with respect to the $i^*$-th expert can be written as
$$
\mathsf E\left[ 
\sum_{t=1}^T\ell_t(x_t)-\sum_{t=1}^T\ell_t(x_t^{(i^*)})
\right]. 
$$
Then, by the update rule (Step~11 of Algorithm~\ref{alg:hedge}), we have for all $i$,
\begin{align*}
w_{t+1}^{(i)}
=
\frac{w_{t}^{(i)}\exp(-\varepsilon\ell_t(x_t^{(i)}))}
{\sum_{j=1}^{N} w_{t}^{(j)}\exp(-\varepsilon\ell_t(x_t^{(j)}))}
=
\frac{w_{1}^{(i)}\exp(-\varepsilon L_t^{(i)})}
{\sum_{j=1}^{N} w_{1}^{(j)}\exp(-\varepsilon L_t^{(j)})}.
\end{align*}
Based on these observations, on the one hand, it can be verified that
\begin{align}\label{eceq:lnw}
\ex\left[\ln W_T\right]&=\ex\left[ \ln W_1 + \sum_{t=2}^T\ln\left( \frac{W_t}{W_{t-1} } \right)\right] \nonumber\\
&=
\ex\left[\sum_{t=1}^T \ex\left[ \ln \left( 
\sum_{i=1}^N w_t^{(i)}\exp\left(-\varepsilon \ell_t(x_t^{(i)}) \right) \right) \middle| \mathscr{F}_t \right]  \right]\nonumber\\
&=
\ex\left[
\sum_{t=1}^T \left(
-\varepsilon \ell_t(x_t) +\ex\left[ \ln\left( \sum_{i=1}^N w_t^{(i)}\exp\left( -\varepsilon \ell_t^{(i)} \right)  \right)   \middle| \mathscr{F}_t \right]
\right)
\right]\nonumber\\
&\leq 
\ex\left[
-\varepsilon\sum_{t=1}^T  \ell_t(x_t)
+\sum_{t=1}^T\ln\left( \sum_{i=1}^N w_t^{(i)}\ex\left[ \exp\left( -\varepsilon \ell_t^{(i)} \right) \Big|   \mathscr{F}_t \right]
\right)
\right],
\end{align}
where the inequality is due to Jensen's inequality. Utilizing the sub-Gaussian property of $\xi_t$ in conjunction with the boundedness of $\tilde U_t$, $\bar U_t$, and $r_t(\cdot)$, we obtain the following bound on
$\ex[ \exp( -\varepsilon \ell_t^{(i)} ) | \mathscr{F}_t ]$.
\begin{align*}
\ex\left[ \exp\left( -\varepsilon \ell_t^{(i)} \right) \Big|   \mathscr{F}_t \right]
&= \exp\left( \varepsilon(x_t^{(i)} - x_t )^{\ts}\bar U_t r_t(\tilde{x}_t)/\delta \right)
\ex\left[  \exp\left( \varepsilon(x_t^{(i)} - x_t )^{\ts}\bar U_t\xi_t/\delta   \right)   \Big|  \mathscr{F}_t \right]\\
&\leq
\exp\left( \varepsilon\sqrt{d}B_v B_r/\delta +\varepsilon^2 d B_v^2\sigma_{\xi}^2/(2\delta^2)   \right)\\
&\leq 
\exp\left( d \varepsilon^2 B_v^2\sigma_{\xi}^2/\delta^2  \right),
\end{align*}
where the last step is from the condition $\delta \leq \varepsilon \sqrt{d}B_v\sigma_{\xi}^2/(2B_r)$, which will be justified later once the choice of $\varepsilon$ is specified.
It follows that 
\begin{equation*}
\eqref{eceq:lnw}\leq 
\ex\left[
-\varepsilon\sum_{t=1}^T  \ell_t(x_t)
+\sum_{t=1}^T\ln\left(\exp\left( d \varepsilon^2 B_v^2\sigma_{\xi}^2/\delta^2  \right)\right)
\right] 
= 
-\varepsilon \ex\left[\sum_{t=1}^T  \ell_t(x_t)\right]
+\left( d \varepsilon^2 B_v^2\sigma_{\xi}^2/\delta^2  \right)T.
\end{equation*}

On the other hand, it holds that
\begin{align*}
\ex\left[ \ln W_T\right]&=\ex\left[  \ln \left(
\sum_{i=1}^{N}w_1^{(i)}\exp\left(-\varepsilon L_T^{(i)}\right)
\right)  \right] \\
&\geq \ex\left[
\ln \left(
\max_{i}\left\{w_1^{(i)}\exp\left(-\varepsilon L_T^{(i)}\right)\right\}
\right)  \right] \\
&= \ex\left[
\ln \left(
\max_i \left\{\exp(\ln w_1^{(i)} - \varepsilon L_T^{(i)})\right\}
\right) \right]  \\
&=
-\varepsilon\ex\left[  \min_i\left\{
L_T^{(i)}+\frac{1}{\varepsilon} \ln\left(\frac{1}{w_1^{(i)}}\right)
\right\} \right].
\end{align*}
Combining the upper and lower bounds on $\ex[\ln W_T]$, we obtain
$$
-\varepsilon \ex\left[\min_i\left\{
L_T^{(i)}+\frac{1}{\varepsilon} \ln\left(\frac{1}{w_1^{(i)}}\right)
\right\}\right]
\leq 
-\varepsilon \ex\left[\sum_{t=1}^T  \ell_t(x_t)\right]
+ \frac{d \varepsilon^2 B_v^2\sigma_{\xi}^2}{\delta^2}T.
$$
Hence for all $i$, we have
\begin{equation}\label{eceq:meta_reg}
\ex\left[\sum_{t=1}^T\ell_t(x_t)-\sum_{t=1}^T\ell_t(x_t^{(i)}) \right]
\leq
\frac{d \varepsilon B_v^2\sigma_{\xi}^2}{\delta^2}T
+
\frac{1}{\varepsilon}\ln\left( \frac{1}{w_1^{(i)}} \right).
\end{equation}
Setting $\varepsilon = \delta/( B_v\sigma_{\xi}\sqrt{d T} )$,  one can verify that this choice satisfies the aforementioned condition $\delta \leq \varepsilon \sqrt{d}B_v\sigma_{\xi}^2/(2B_r)$ whenever  $2B_r\sqrt{T}\leq \sigma_{\xi}$.
The right-hand side of \eqref{eceq:meta_reg} then becomes
$$
\frac{ \sigma_{\xi}B_v\sqrt{d}}{\delta}\sqrt{T}
\left(
1+\ln\left( \frac{1}{w_1^{(i)}} \right)
\right).
$$
Focusing again on the $i^*$-th expert and recalling the definition of $w_1^{(i^*)}$, we get that
\begin{equation}\label{eceq:meta_reg_bound}
\eqref{eceq:meta_reg}
\leq 
\frac{ \sigma_{\xi}B_v\sqrt{d}}{\delta}\sqrt{T}
\left(
1+2\ln\left(i^{*}+1 \right)
\right)
\leq
\frac{ \sigma_{\xi}B_v\sqrt{d}}{\delta}\sqrt{T}
\left(
1+2\ln\left( \log_2\lceil(\sqrt{d}T/V_T)^{2/3}\rceil+2 \right)
\right),
\end{equation}
where the last inequality is derived from \eqref{eceq:i_star}, which implies that $i^*\leq \log_2\lceil(\sqrt{d}T/V_T)^{2/3}\rceil+1$.

Finally, by combining \eqref{eceq:brace}, \eqref{eceq:exp_reg_bound}, and \eqref{eceq:meta_reg_bound}, and recalling the $L$-smoothness of $r_t(\cdot)$, we obtain the following bound on the dynamic regret:
\begin{align*}
\mathsf E\left[\sum_{t=1}^T
\left( r_t(x^*_t)-r_t(\tilde x_t) \right)
\right]&=\mathsf E\left[\sum_{t=1}^T
\left( r_t(x^*_t)-r_t(x_t) \right)
\right]+\mathsf E\left[\sum_{t=1}^T
\left(r_t(x_t)-r_t(x_t+\delta \tilde U_t) \right)
\right]\\
&\leq 
\mathsf E\left[
\sum_{t=1}^T G_t^{\mathtt T} (x^*_t-x_t^{(i^*)})
\right]
+\mathsf E\left[
\sum_{t=1}^T G_t^{\mathtt T} (x_t^{(i^*)}-x_t)
\right]
+2 C_1 \delta^p \sqrt{d}T+\frac{LC_u}{2}\delta^2T
\\
&\leq
\left(4d^{-1/3}\sqrt{\frac{B}{\alpha} }G_{\delta}
+2d^{1/3}\left( C_1\delta^p+\sqrt{C_r}\right)\right)
T^{2/3}V_T^{1/3}\\
&\quad
+\frac{ \sigma_{\xi}B_v\sqrt{d}}{\delta}\sqrt{T}
\left(
1+2\ln\left( \log_2\lceil(\sqrt{d}T/V_T)^{2/3}\rceil+2 \right)
\right) + \left(2 C_1 \delta^p \sqrt{d}+\frac{LC_u}{2}\delta^2\right)T.
\end{align*}
\Halmos  
\endproof

\newpage

\section{Algorithm Performance under Different Variation Levels} \label{sec:econnumerical}
\subsection{V=0}

\begin{figure}[H]
\centering
\includegraphics[width=.5\linewidth]{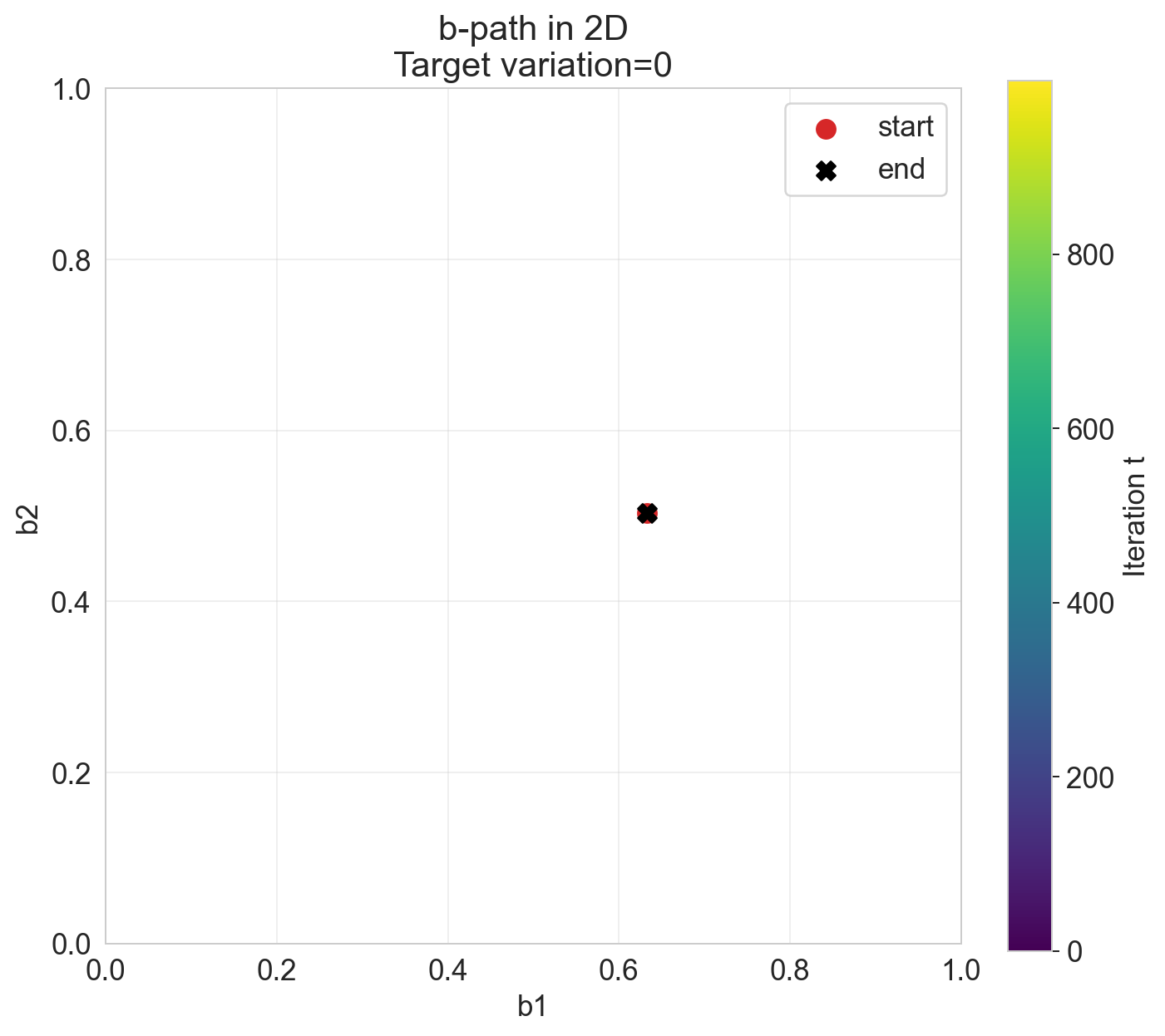}
\caption{Variation=0, b path}
\label{fig:placeholder}
\end{figure}

\begin{figure}[H]
\centering
\includegraphics[width=\linewidth]{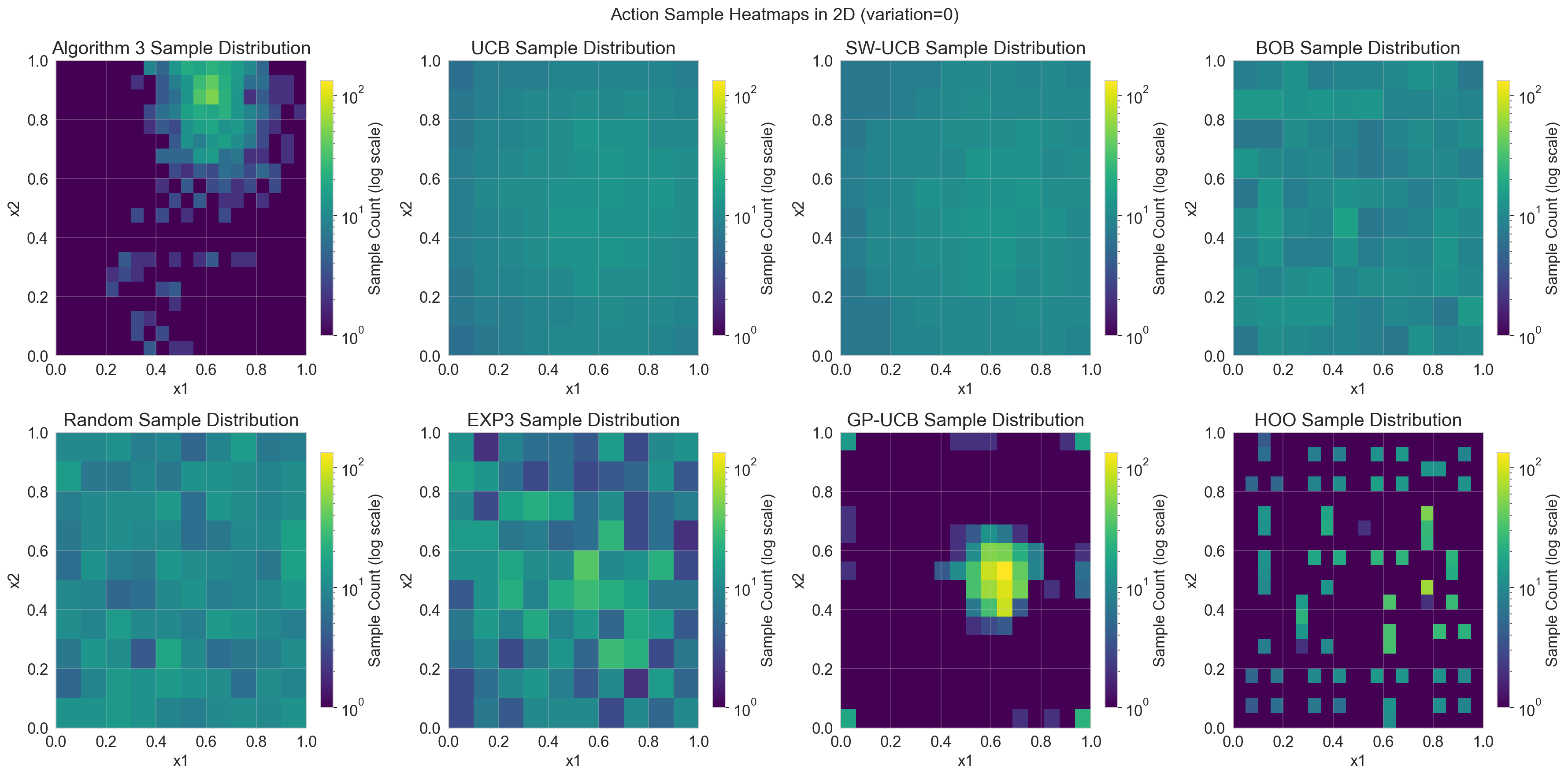}
\caption{Variation=0, sample heatmaps}
\label{fig:placeholder}
\end{figure}

\begin{figure}[H]
\centering
\includegraphics[width=\linewidth]{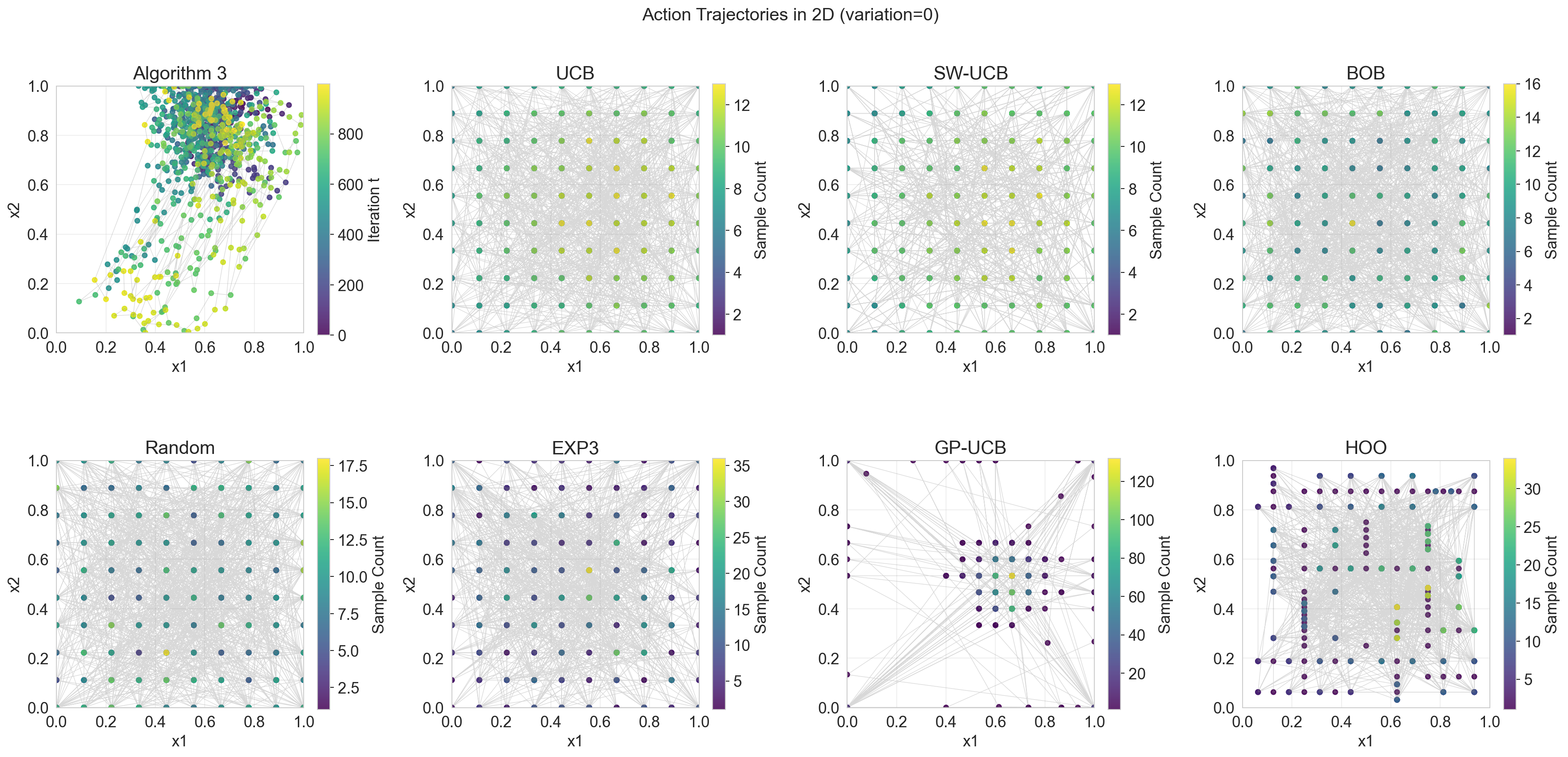}
\caption{Variation=0, action trajectories}
\label{fig:placeholder}
\end{figure}

\begin{figure}[H]
\centering
\includegraphics[width=.8\linewidth]{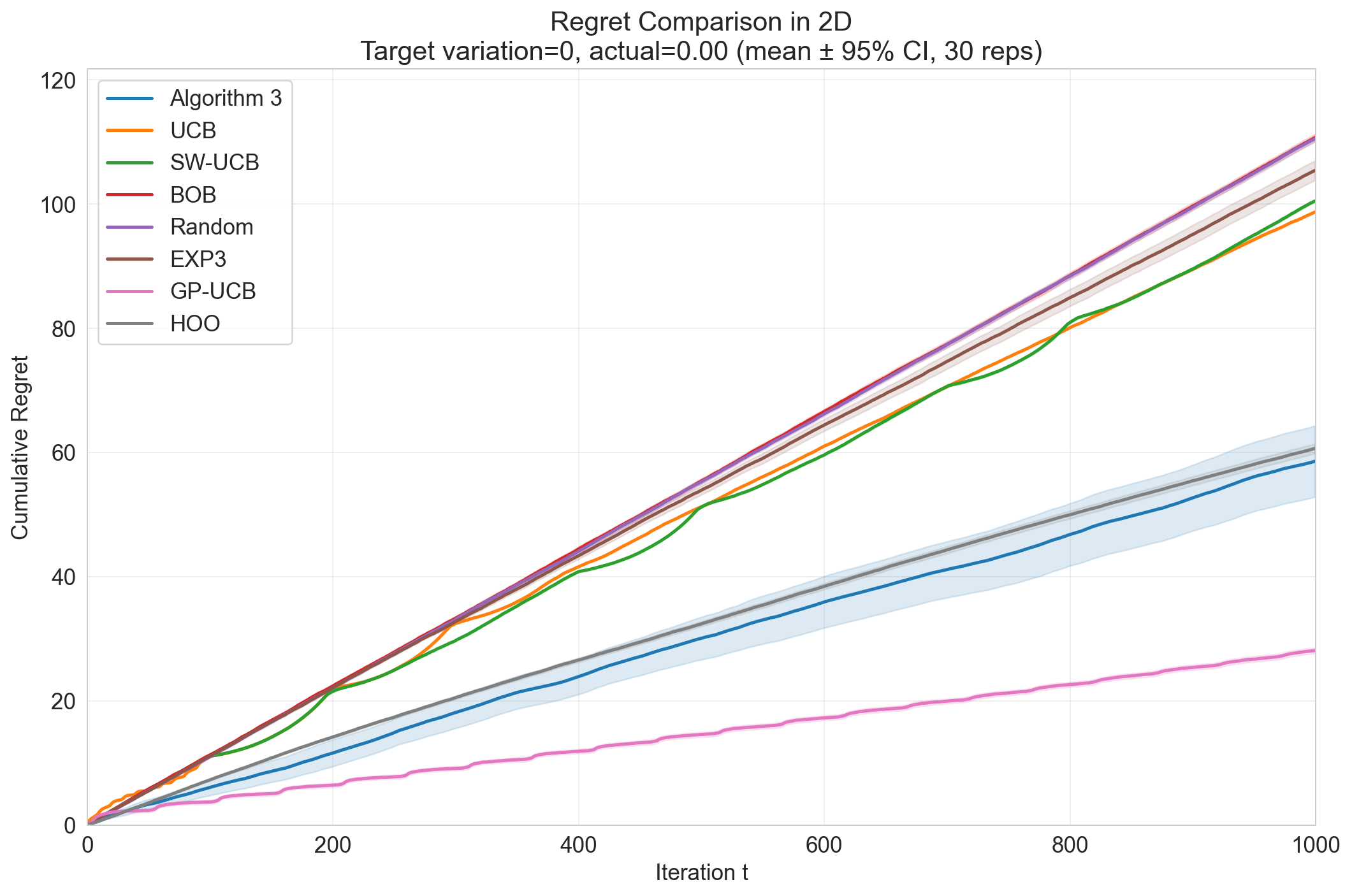}
\caption{Variation=0, regret }
\label{fig:placeholder}
\end{figure}

\newpage

\subsection{V=20}

\begin{figure}[H]
\centering
\includegraphics[width=.5\linewidth]{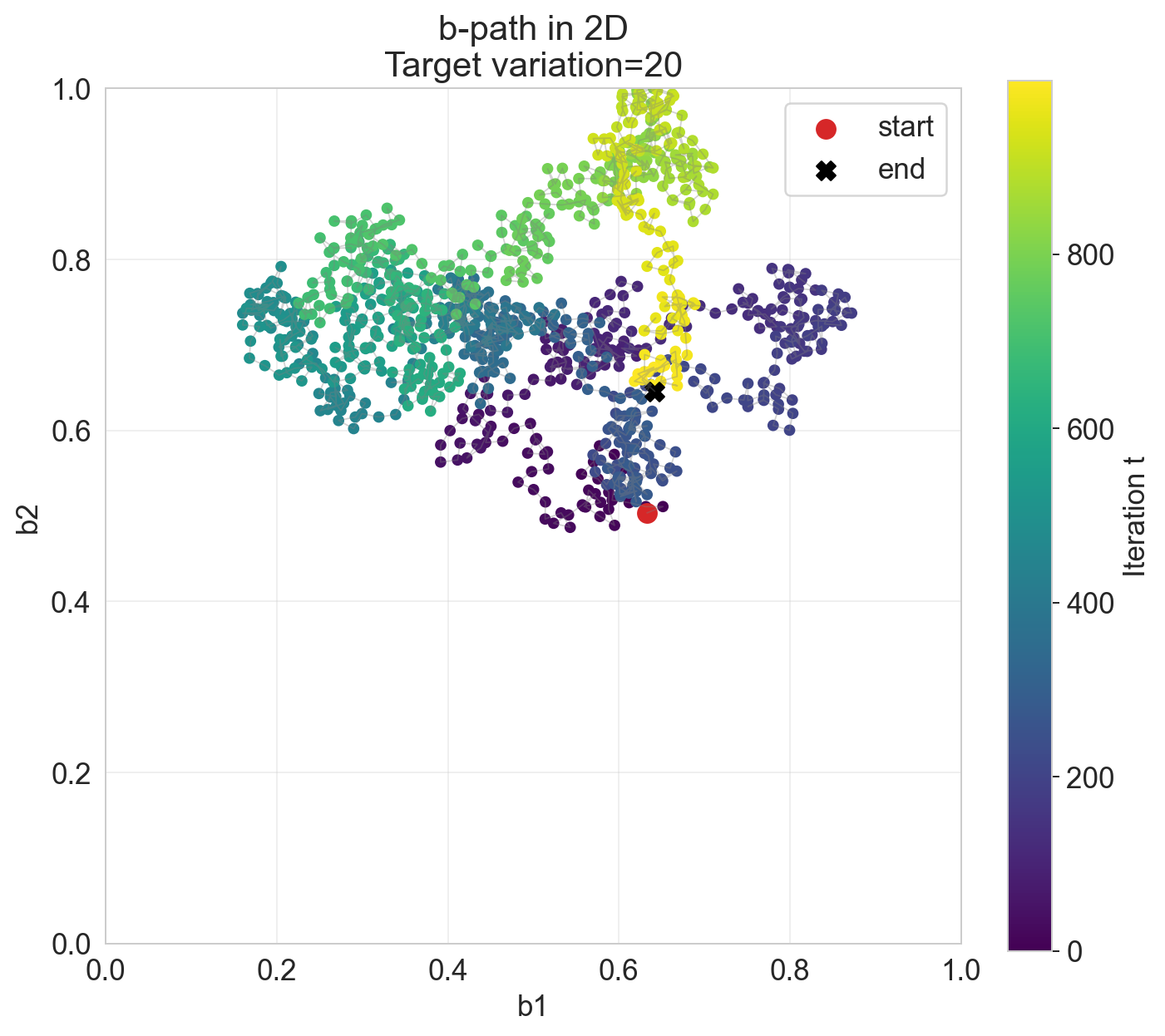}
\caption{Variation=20, b path}
\label{fig:placeholder}
\end{figure}

\begin{figure}[H]
\centering
\includegraphics[width=\linewidth]{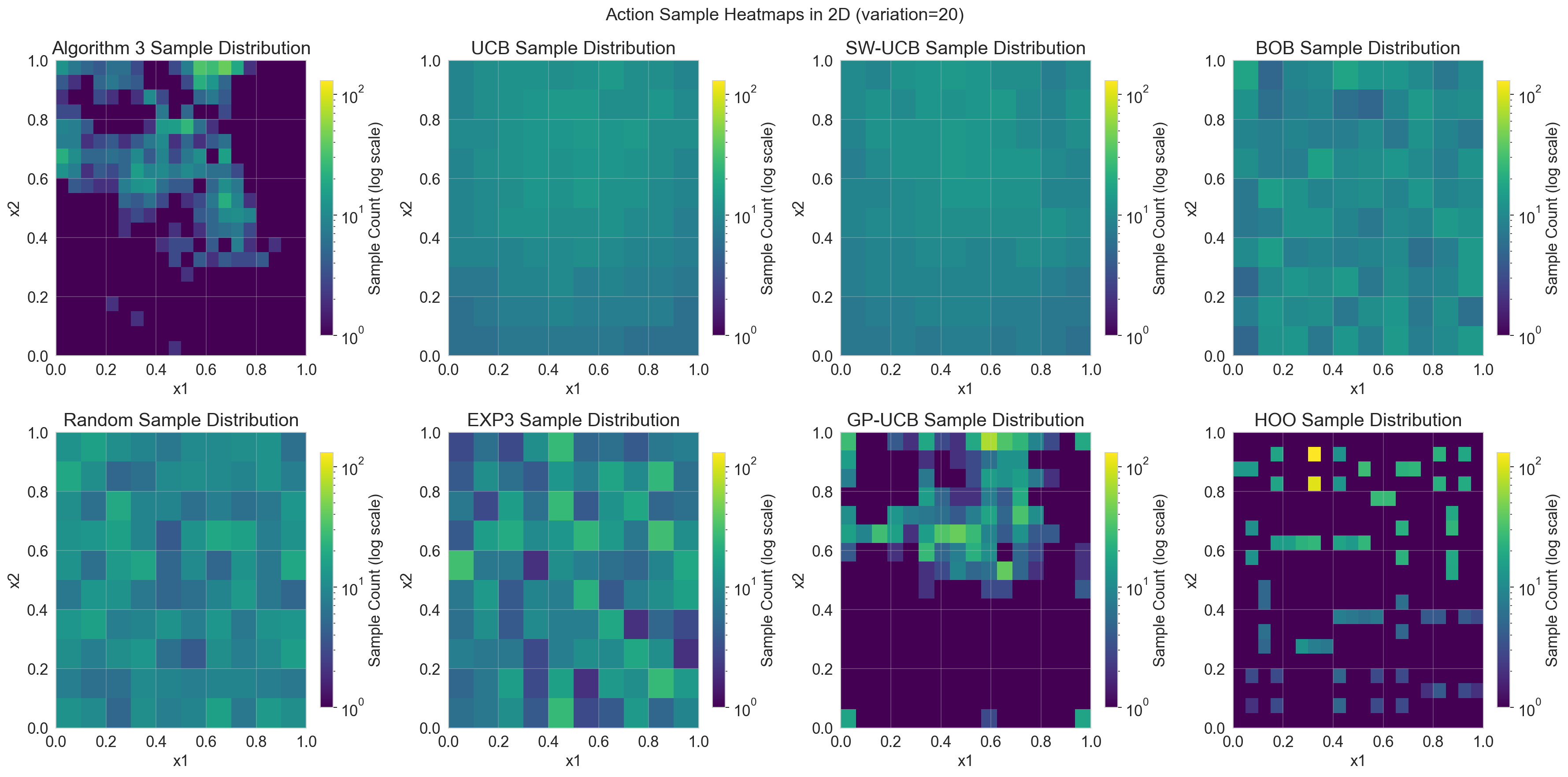}
\caption{Variation=20, sample heatmaps}
\label{fig:placeholder}
\end{figure}

\begin{figure}[H]
\centering
\includegraphics[width=\linewidth]{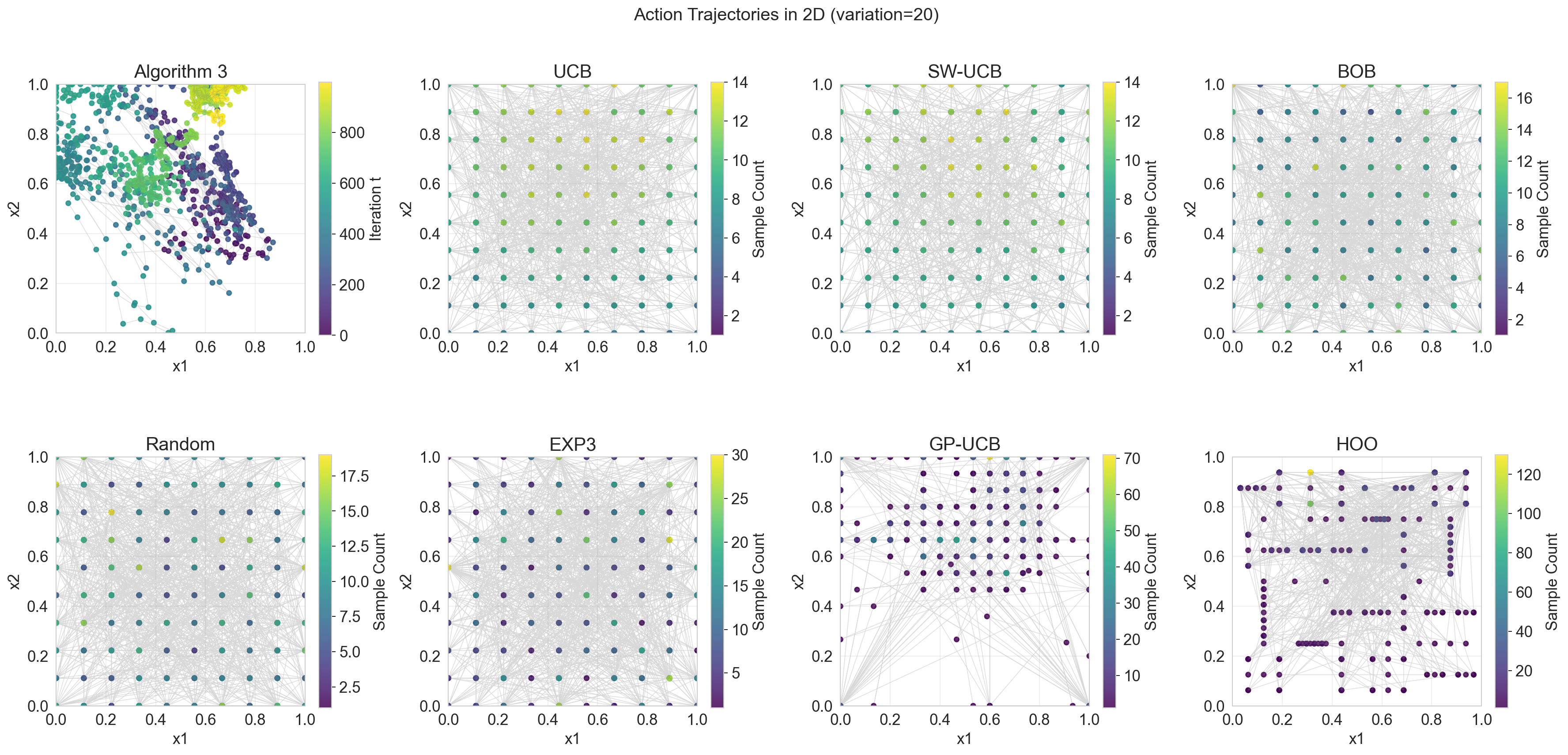}
\caption{Variation=20, action trajectories}
\label{fig:placeholder}
\end{figure}

\begin{figure}[H]
\centering
\includegraphics[width=.8\linewidth]{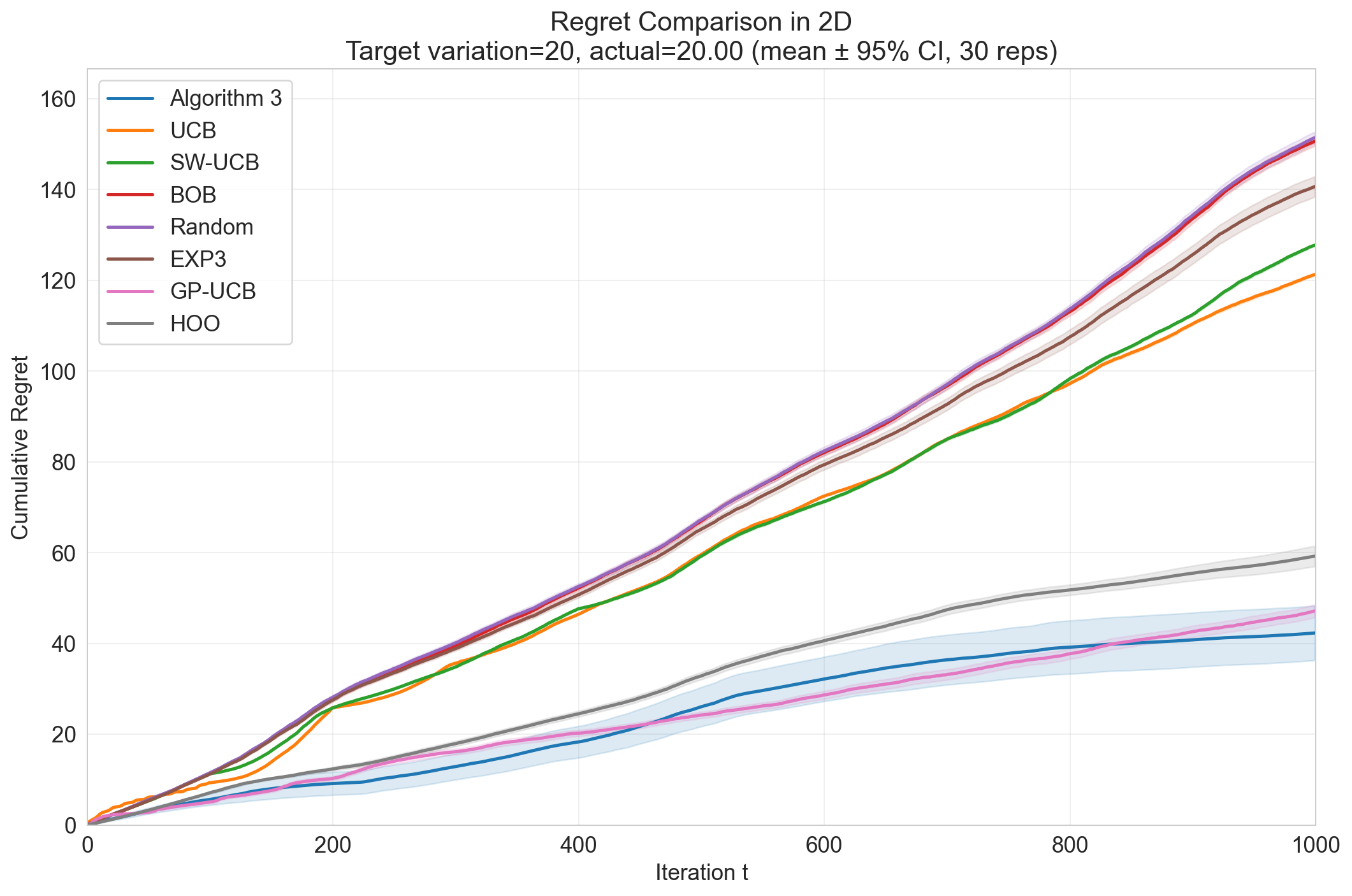}
\caption{Variation=20, regret }
\label{fig:placeholder}
\end{figure}

\newpage

\subsection{V=30}

\begin{figure}[H]
\centering
\includegraphics[width=.5\linewidth]{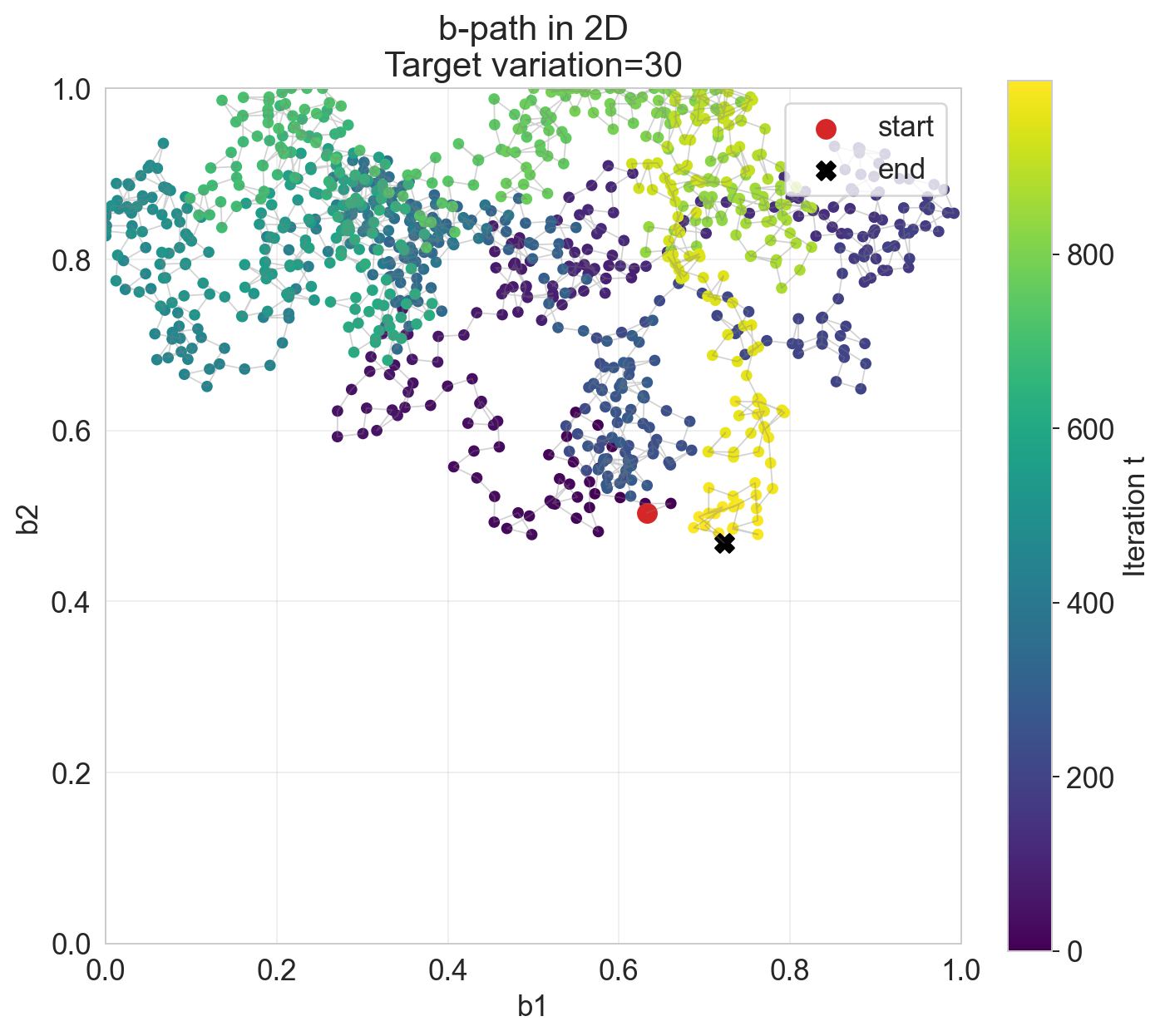}
\caption{Variation=30, b path}
\label{fig:placeholder}
\end{figure}

\begin{figure}[H]
\centering
\includegraphics[width=\linewidth]{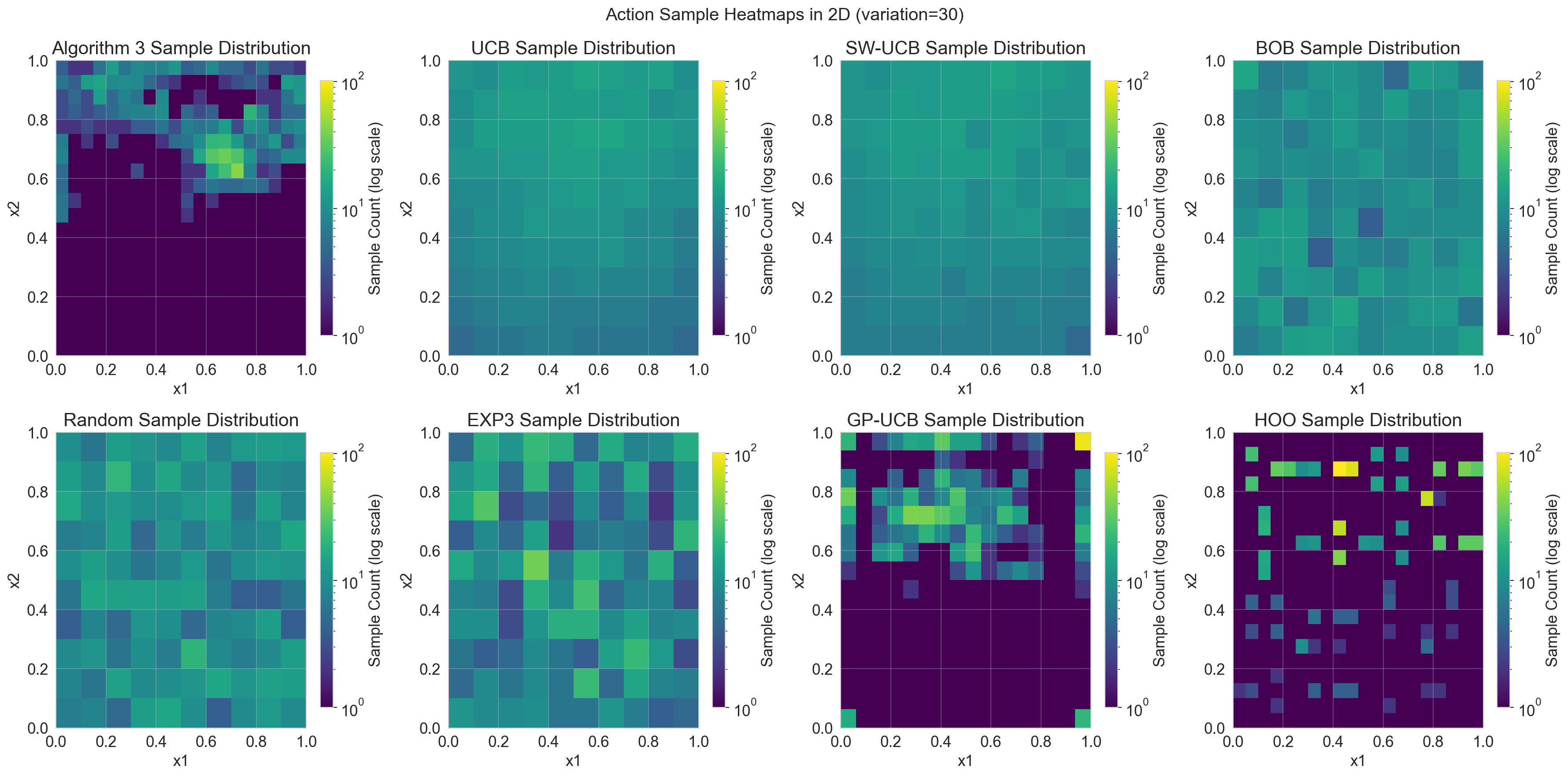}
\caption{Variation=30, sample heatmaps}
\label{fig:placeholder}
\end{figure}

\begin{figure}[H]
\centering
\includegraphics[width=\linewidth]{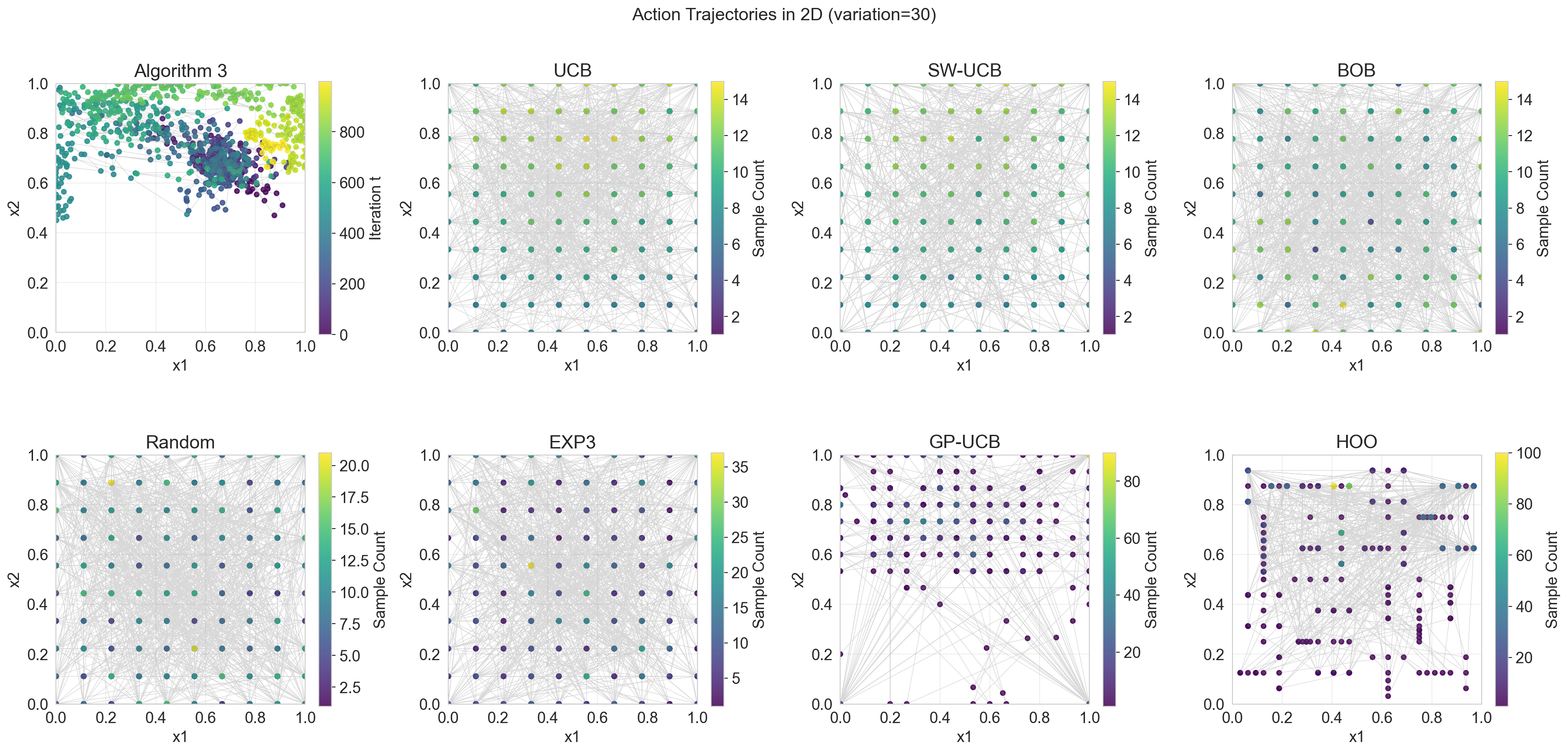}
\caption{Variation=30, action trajectories}
\label{fig:placeholder}
\end{figure}

\begin{figure}[H]
\centering
\includegraphics[width=.8\linewidth]{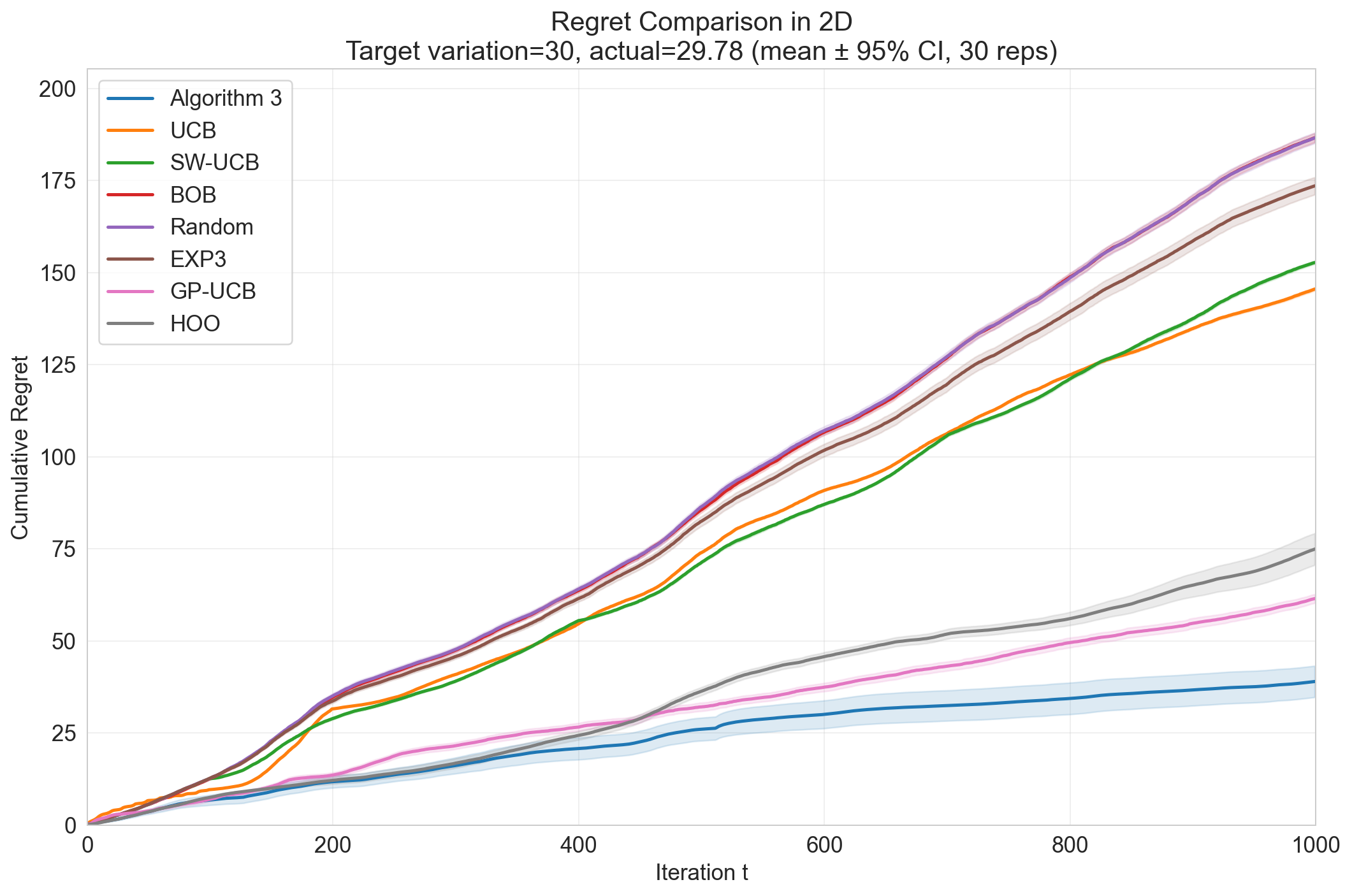}
\caption{Variation=30, regret }
\label{fig:placeholder}
\end{figure}

\end{document}